%% file: main.tex
\documentclass{article}

\usepackage{microtype}
\usepackage{graphicx}
\usepackage{subfigure}
\usepackage{booktabs} 

\usepackage{hyperref}

\usepackage[accepted]{icml2021}

\usepackage[utf8]{inputenc} 
\usepackage[T1]{fontenc}    
\usepackage{hyperref}       
\usepackage{url}            
\usepackage{booktabs}       
\usepackage{amsfonts}       
\usepackage{nicefrac}       
\usepackage{microtype}      
\usepackage{wrapfig}
\usepackage[font=small,labelfont=bf]{caption}
\usepackage{graphicx}
\usepackage{subfigure}
\usepackage{multirow}
\usepackage{pgfplots}

\usepackage{amssymb,amsthm,amsmath}
\usepackage{mathtools}
\usepackage{bm}
\usepackage[inline]{enumitem}
\usepackage{tikz}
\usepackage{algorithm}
\usepackage{xpatch}
\usepackage{outlines}
\usepackage{xspace}
\usepackage{bbm}

\makeatletter
\xpatchcmd{\algorithmic}{\itemsep\z@}{\itemsep=0.5ex plus1pt}{}{}
\makeatother

\newtheorem{proposition}{Proposition}
\newtheorem{lemma}{Lemma}

\input{commands.tex}
\input{math_commands.tex}

\icmltitlerunning{Learning Binary Decision Trees by Argmin Differentiation}

\begin{document}

\twocolumn[
\icmltitle{Learning Binary Decision Trees by Argmin Differentiation}





\begin{icmlauthorlist}
\icmlauthor{Valentina Zantedeschi}{inria,ucl}
\icmlauthor{Matt J. Kusner}{ucl}
\icmlauthor{Vlad Niculae}{am}
\end{icmlauthorlist}

\icmlaffiliation{inria}{Inria, Lille - Nord Europe research centre}
\icmlaffiliation{ucl}{University College London, Centre for Artificial Intelligence}
\icmlaffiliation{am}{Informatics Institute, University of Amsterdam}

\icmlcorrespondingauthor{V Zantedeschi}{vzantedeschi@gmail.com}
\icmlcorrespondingauthor{M Kusner}{m.kusner@ucl.ac.uk}
\icmlcorrespondingauthor{V Niculae}{v.niculae@uva.nl}


\vskip 0.3in
]



\printAffiliationsAndNotice{}  


\begin{abstract}
We address the problem of learning binary decision trees that partition data for some downstream task. We propose to learn discrete parameters (i.e., for tree traversals and node pruning) and continuous parameters (i.e., for tree split functions and prediction functions) \emph{simultaneously} using argmin differentiation. We do so by sparsely relaxing a mixed-integer program for the discrete parameters, to allow gradients to pass through the program to continuous parameters. We derive customized algorithms to efficiently compute the forward and backward passes. This means that our tree learning procedure can be used as an (implicit) layer in arbitrary deep networks, and can be optimized with arbitrary loss functions. We demonstrate that our approach produces binary trees that are competitive with existing single tree and ensemble approaches, in both supervised and unsupervised settings. Further, apart from greedy approaches (which do not have competitive accuracies), our method is faster to train than all other tree-learning baselines we compare with.
The code for reproducing the results is available at \url{https://github.com/vzantedeschi/LatentTrees}.
\end{abstract}

\section{Introduction}
\input{sections/introduction.tex}

\section{Related Work}
\input{sections/related.tex}

\section{Method}
\input{sections/methods.tex}
\section{Experiments}
\input{sections/experiments.tex}

\section{Discussion}
\input{sections/discussion.tex}

\section*{Acknowledgements}
The authors are thankful to Mathieu Blondel and Caio Corro
for useful suggestions.
They acknowledge the support of Microsoft AI for Earth grant as well as The Alan Turing Institute in the form of Microsoft Azure computing credits.
VN acknowledges support from
the European Research Council (ERC StG DeepSPIN 758969)
and the Funda\c{c}\~ao para a Ci\^encia e Tecnologia
through contract UIDB/50008/2020
while at the Instituto de Telecomunica\c{c}\~{o}es.

\bibliography{bib}
\bibliographystyle{icml2021}

\onecolumn
\appendix
\section{Appendix}
\input{sections/appendix.tex}

\end{document}

%% file: commands.tex
\usetikzlibrary{cd,arrows,calc,shapes,positioning}
\tikzstyle{obs} = [circle,fill=white,draw=black,inner sep=1pt,minimum size=20pt,font=\fontsize{10}{10}\selectfont,node distance=1,thick]
\tikzstyle{latent} = [obs,dotted]
\tikzstyle{protected} = [obs,text=Orange,draw=Orange]
\tikzstyle{unfair} = [obs,text=BrickRed,draw=BrickRed]
\tikzstyle{target} = [obs,text=MidnightBlue,draw=MidnightBlue]
\tikzstyle{feature} = [obs,text=ForestGreen,draw=ForestGreen]

\newcommand{\B}[1]{\mathbf{#1}}

\newcommand{\bR}{\ensuremath \mathbb{R}}

\newcommand{\cF}{\ensuremath \mathcal{F}}

\newcommand{\cT}{\ensuremath \mathcal{T}}

\newcommand{\xb}{\ensuremath \B{x}}
\newcommand{\yb}{\ensuremath \B{y}}
\newcommand{\ab}{\ensuremath \B{a}}

\newcommand{\zb}{\ensuremath \B{z}}
\newcommand{\qb}{\ensuremath \B{q}}

\DeclareMathOperator*{\argmax}{arg\,max}
\DeclareMathOperator*{\argmin}{arg\,min}

\newcommand\pfrac[2]{\frac{\partial #1}{\partial #2}}

\newcommand\hlf{\nicefrac{1}{2}}
\newcommand*{\eg}{\textit{e.g.}\@\xspace}
\newcommand*{\ie}{\textit{i.e.}\@\xspace}
\newcommand*{\wrt}{\textit{w.r.t.}\@\xspace}

\definecolor{CRuby}{HTML}{D81159}
\definecolor{CMagenta}{HTML}{8F2D56}
\definecolor{CGreen}{HTML}{218380}
\definecolor{CYellow}{HTML}{FBB13C}
\definecolor{CBlue}{HTML}{73D2DE}

%% file: math_commands.tex
\usepackage{amsmath,amsfonts,bm}

\def\1{\bm{1}}

\DeclareMathAlphabet{\mathsfit}{\encodingdefault}{\sfdefault}{m}{sl}
\SetMathAlphabet{\mathsfit}{bold}{\encodingdefault}{\sfdefault}{bx}{n}



%% file: sections/introduction.tex
Learning discrete structures from unstructured data is extremely useful for a wide variety of real-world problems \citep{gilmer2017neural,kool2018attention,yang2018deep}.
One of the most computationally-efficient, easily-visualizable discrete structures that are able to represent complex functions are \emph{binary trees}. For this reason, there has been a massive research effort on how to learn such binary trees since the early days of machine learning \citep{payne1977algorithm,breiman1984classification,bennett1992decision,bennett1996optimal}.
Learning binary trees has historically been done in one of three ways.
The first is via \emph{greedy optimization}, which includes popular decision-tree methods such as classification and regression trees
\citep[CART, ][]{breiman1984classification} and ID3 trees \citep{quinlan1986induction}, among many others. These methods optimize a splitting criterion for each tree node, based on the data routed to it.
The second set of approaches are based on \emph{probabilistic
relaxations}~\citep{irsoy2012soft,yang2018deep,DBLP:conf/icml/0001PMTM20},
optimizing all splitting parameters at once via gradient-based methods, by relaxing hard branching decisions into branching probabilities.
The third approach optimizes trees using \emph{mixed-integer programming} \citep[MIP, ][]{bennett1992decision,bennett1996optimal}. This jointly optimizes all continuous and discrete parameters to find globally-optimal trees.\footnote{Here we focus on learning single trees instead of tree ensembles; our work easily extends to ensembles.}


\begin{figure}[t!]
    \centering
    {\includegraphics[width=\columnwidth]{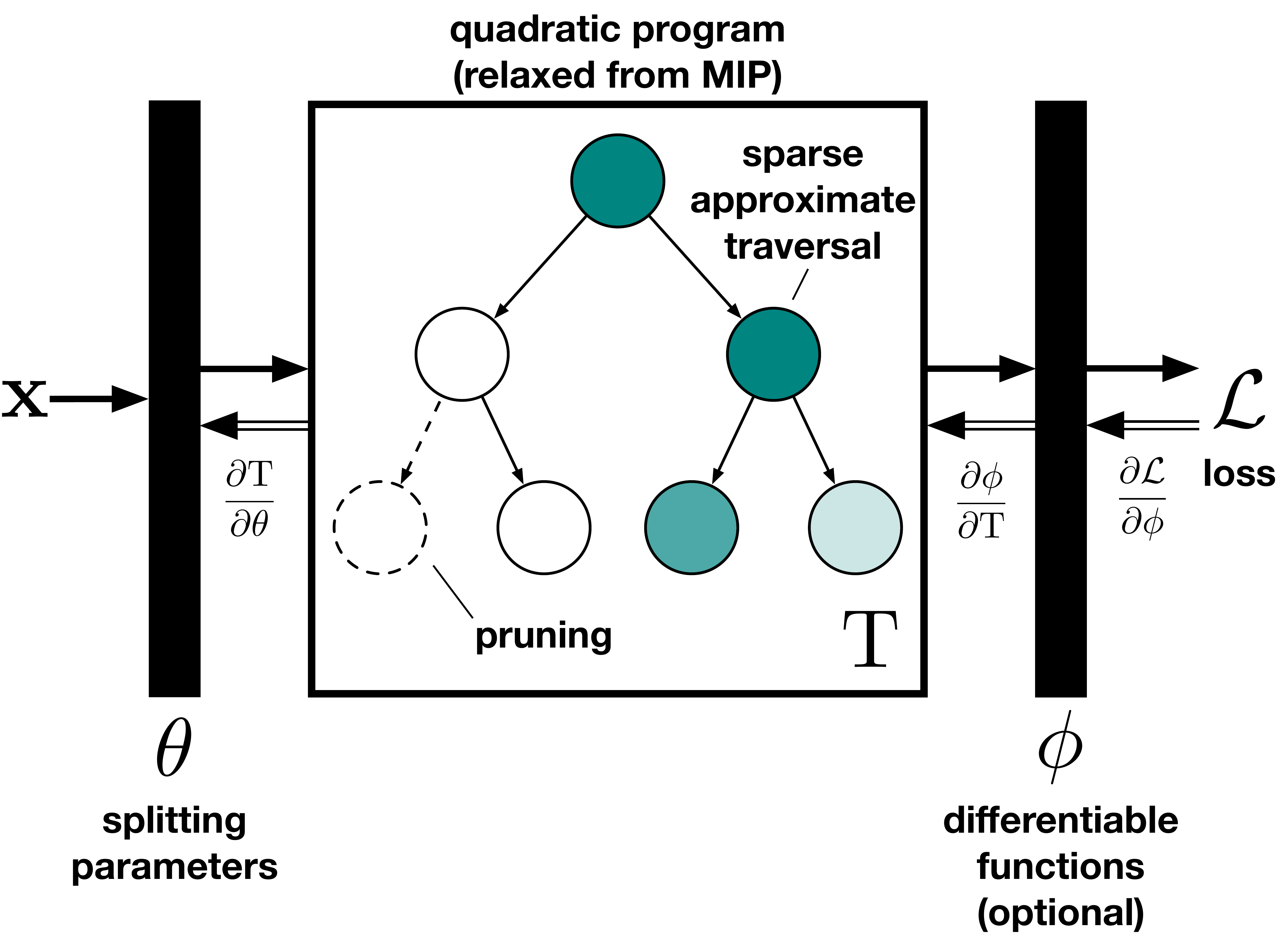}
    }
    \caption{A schematic of the approach.}
    \label{fig:overall}
\end{figure}

Each of these approaches has their shortcomings. First, greedy optimization is
generally suboptimal: tree splitting criteria are even intentionally crafted to be different than the global tree loss, as the global loss may not encourage tree growth \citep{breiman1984classification}. Second, probabilistic relaxations: (a) are rarely sparse, so inputs contribute to all branches even if they are projected to hard splits after training; (b) they do not have principled ways to prune trees, as the distribution over pruned trees is often intractable. Third, MIP approaches, while optimal, are only computationally tractable on training datasets with thousands of inputs \citep{bertsimas2017optimal}, and do not have well-understood out-of-sample generalization guarantees.


In this paper we present an approach to binary tree learning based on sparse relaxation and argmin differentiation (as depicted in Figure~\ref{fig:overall}).
Our main insight is that by quadratically relaxing an MIP that learns the discrete parameters of the tree (input traversal and node pruning), we can differentiate through it to simultaneously learn the continuous parameters of splitting decisions. This allows us to leverage the superior generalization capabilities of stochastic gradient optimization to learn decision splits, and gives a principled approach to learning tree pruning.
Further, we derive customized algorithms to compute the forward and backward passes through this program that are much more efficient than generic approaches \citep{optnet,cvxpylayers2019}.
The resulting learning method 
can learn trees with any differentiable splitting functions to minimize any differentiable loss, both tailored to the data at hand. 

\paragraph{Notation.} We denote scalars, vectors, and sets, as $x$, $\xb$,
and $\mathcal{X}$, respectively.
A binary tree is a set $\mathcal{T}$ containing nodes $t \in \mathcal{T}$
with the additional structure that each node has at most one left child and at
most one right child, and each node except the unique root has exactly one
parent.
We denote the $l_2$ norm of a vector by $\|\xb\| \coloneqq\left(\sum_i
x_i^2\right)^{\frac{1}{2}}$. The unique projection of point $\xb \in \mathbb{R}^d$ onto a convex set $C \subset
\mathbb{R}^d$ is $\operatorname{Proj}_{C}(\xb) \coloneqq \argmin_{\yb \in C} \|\yb - \xb\|$.
In particular, projection onto an interval is given by
$\operatorname{Proj}_{[a, b]}(x) = \max(a, \min(b, x))$.

%% file: sections/related.tex
The paradigm of \emph{binary tree learning} has the goal of finding a tree that iteratively splits data into meaningful, informative subgroups, guided by some criterion.
Binary tree learning appears in a wide variety of problem settings across machine learning. We briefly review work in two learning settings where latent tree learning plays a key role: 1. \emph{Classification/Regression}; and 2. \emph{Hierarchical clustering}. 
Due to the generality of our setup (tree learning with arbitrary split functions, pruning, and downstream objective), our approach can be used to learn trees in any of these settings. Finally, we detail how parts of our algorithm are inspired by recent work in isotonic regression.

\paragraph{Classification/Regression.}
Decision trees for classification and regression have a storied history, with early popular methods that include classification and regression trees
\citep[CART, ][]{breiman1984classification}, ID3 \citep{quinlan1986induction}, and C4.5 \citep{quinlan1993c4}. While powerful, these methods are greedy: they sequentially identify `best' splits as those which optimize a split-specific score (often different from the global objective). As such, learned trees are likely sub-optimal for the classification/regression task at hand.
To address this, \citet{carreira2018alternating} proposes an alternating algorithm for refining the structure and decisions of a tree so that it is smaller and with reduced error, however still sub-optimal.
Another approach is to probabilistically relax the discrete splitting decisions of the tree \citep{irsoy2012soft,yang2018deep,tanno2019adaptive}. This allows the (relaxed) tree to be optimized \wrt the overall objective using gradient-based techniques, with known generalization
benefits \citep{hardt2016train,hoffer2017train}. Variations on this approach aim at learning tree ensembles termed `decision forests' \citep{kontschieder2015deep,lay2018random,popov2019neural,DBLP:conf/icml/0001PMTM20}. The downside of the probabilistic relaxation approach is that there is no principled way to prune these trees  as inputs pass through all nodes of the tree with some probability.
A recent line of work has explored mixed-integer program (MIP) formulations for learning decision trees.
Motivated by the billion factor speed-up in MIP in the last 25 years,
\citet{rudin2018learning} proposed a mathematical programming approach for learning provably optimal decision lists~\citep[one-sided decision trees;][]{letham2015interpretable}.
This resulted in a line of recent follow-up works extending the problem to binary decision trees~\citep{hu2019optimal,lin2020generalized} by adapting the efficient discrete optimization algorithm \citep[CORELS, ][]{angelino2017learning}.
Related to this line of research, \citet{bertsimas2017optimal} and its follow-up works~\citep{gunluk2018optimal,aghaei2019learning,verwer2019learning,aghaei2020learning} phrased the objective of CART as an MIP that could be solved exactly.
Even given this consistent speed-up all these methods are only practical on datasets with at most thousands of inputs \citep{bertsimas2017optimal} and with non-continuous features.
Recently, \citet{DBLP:conf/nips/ZhuMPNK20} addressed these tractability concerns principally with a data selection mechanism that preserves information.
Still, the out-of-sample generalizability of MIP approaches is not well-studied, unlike stochastic gradient descent learning.

\paragraph{Hierarchical clustering.}
Compared to standard flat clustering, hierarchical clustering provides a structured organization of unlabeled data in the form of a tree. To learn such a clustering the vast majority of methods are greedy and work in one of two ways: 1. \emph{Agglomerative}: a `bottom-up' approach that starts each input in its own cluster and iteratively merges clusters; and 2. \emph{Divisive}: a `top-down' approach that starts with one cluster and recusively splits clusters  \citep{zhang1997birch,widyantoro2002incremental,krishnamurthy2012efficient,dasgupta2016cost,kobren2017hierarchical,moseley2017approximation}. These methods suffer from similar issues as greedy approaches for classification/regression do: they may be sub-optimal for optimizing the overall tree. Further they are often computationally-expensive due to their sequential nature. Inspired by approaches for classification/regression, recent work has designed probabilistic relaxations for learning hierarchical clusterings via gradient-based methods \citep{monath2019gradient}.

Our work takes inspiration from both the MIP-based and gradient-based
approaches. Specifically, we frame learning the discrete tree parameters as an MIP, which we sparsely relax to allow continuous parameters to be optimized by argmin differentiation methods.

\paragraph{Argmin differentiation.}
Solving an optimization problem as a differentiable module within a parent problem tackled with gradient-based optimization methods is known as argmin differentiation, an instance of bi-level optimization \citep{colson2007overview,gould}. This situation arises in as diverse scenarios as hyperparameter optimization
\citep{pedregosa2016hyperparameter}, meta-learning \citep{metalearning}, or structured prediction \citep{ves,domke,sparsemap}.
General algorithms for quadratic \cite{optnet} and disciplined convex
programming \citep[Section 7,][]{amos2019differentiable,cvxpylayers2019,diffcone} have
been given, as well as expressions for more specific cases like isotonic
regression \citep{djolonga}. Here, by taking advantage of the structure of the decision tree induction problem, we obtain a direct, efficient algorithm.

\paragraph{Latent parse trees.} Our work resembles but should not be confused with the latent
parse tree literature in natural language processing
\citep{yogatama2017learning,lapata,choi2018learning,williams2018latent,sparsemapcg,caio-iclr,caio-acl,maillard2019jointly,kimcompound,kimrnng}.
This line of work has a different goal than ours: to induce a tree for each
individual data point (e.g., sentence). In contrast, our work aims to learn a
single tree, for all instances to pass through.

\paragraph{Isotonic regression.} Also called monotonic regression, isotonic regression \citep{barlowstatistical} constrains the regression function to be non-decreasing/non-increasing. This is useful if one has prior knowledge of such monotonicity (\eg, the mean temperature of the climate is non-decreasing). A classic algorithm is pooling-adjacent-violators (PAV),
which optimizes the pooling of adjacent points that violate the monotonicity constraint
\citep{barlowstatistical}. This initial algorithm has been generalized and incorporated into convex programming frameworks (see \citet{mair2009isotone} for an excellent summary of the history of isotonic regression and its extensions). Our work builds off of the generalized PAV (GPAV) algorithm of \citet{yu2016exact}.

%% file: sections/methods.tex
Given inputs $\{\xb_i \in \mathcal{X} \}_{i=1}^n$,  
our goal is to learn a latent binary decision tree $\cT$ with maximum depth $D$. This tree sends each input $\xb$ through branching nodes to a specific leaf node  in the tree. Specifically, all branching nodes $\cT_B \subset \cT$ split an input $\xb$ by forcing it to go to its left child if
$s_\theta(\xb) < 0$,
and right otherwise. There are three key parts of the tree that need to be
identified: 1. The \emph{continuous} parameters $\theta_t \in \mathbb{R}^d$ that
describe how $s_{\theta_t}(.)$ splits inputs at every node $t$; 2. The
\emph{discrete} paths $\zb$ made by each input $\xb$ through the tree; 3. The
\emph{discrete} choice $a_t$ of whether a node $t$ should be active or pruned,
i.e. inputs should reach/traverse it or not. We next describe how to represent each of these.


\setlength{\textfloatsep}{8pt}
\begin{figure}
    \centering
    {\includegraphics[width=\columnwidth]{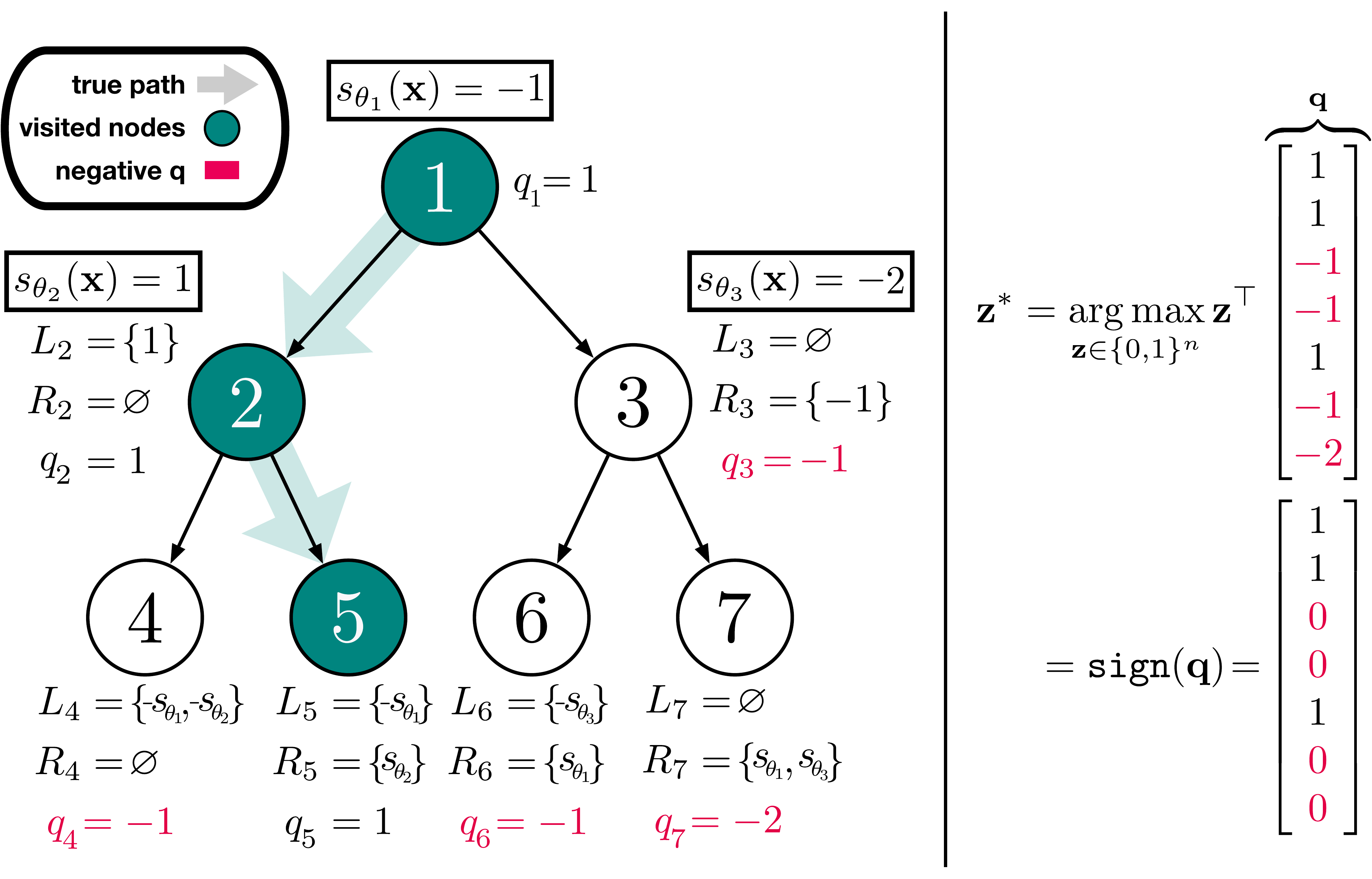}
    }
    \caption{A depiction of equation~\eqref{eq:lp_no_prune} for optimizing tree traversals given the tree parameters $\bm{\theta}$.}
    \label{fig:eq_1}
\end{figure}
\subsection{Tree Traversal \& Pruning Programs}
Let the splitting functions of a complete tree $\{s_{\theta_t}: \mathcal{X} \to \mathbb{R}\}_{t=1}^{|\cT_B|}$ be fixed. The path through which a point $\xb$ traverses the tree can be encoded in
a vector $\zb \in \{0,1\}^{|\cT|}$, where each component $z_t$ indicates whether $\xb$ reaches node $t \in \cT$.
The following integer linear program (ILP) is maximized by a vector $\zb$ that describes \emph{tree traversal}:
\begin{align}
    \max_{\zb} \;\;& \zb^\top \qb \label{eq:lp_no_prune} \\
    \mbox{s.t. } \forall t \in \cT \setminus \{ \text{root} \},\;\;& q_{t} = \min \{ R_{t} \cup L_{t} \} \nonumber \\
    & R_{t} = \{ s_{\theta_{t'}}(\xb) \;\;\;\!\mid \;\; \forall t' \in A_R(t)\} \nonumber \\ 
    & L_{t} = \{ -s_{\theta_{t'}}(\xb) \mid \;\; \forall t' \in A_L(t) \} \nonumber \\ 
    & z_{t} \in \{0, 1\}, \nonumber
\end{align}
where we fix $q_1\!=\!z_1\!=\!1$ (i.e., the root). Here $A_L(t)$ is the set of ancestors of node $t$ whose left child must be followed to get to $t$, and similarly for $A_R(t)$. 
The quantities $q_{t}$ (where $\qb \!\in\! \mathbb{R}^{|\cT|}$ is the tree-vectorized version of $q_{t}$) describe the `reward' of sending $\xb$ to node $t$. This is based on how well the splitting functions leading to $t$ are satisfied ($q_{t}$ is positive if all splitting functions are satisfied and negative otherwise). This is visualized in Figure~\ref{fig:eq_1} and formalized below.

\begin{lemma}\label{lemma:tree}
For an input $\xb$ and binary tree $\cT$ with splits $\{s_{\theta_{t}}, \forall t \in \cT_B\}$, the ILP in eq.~\eqref{eq:lp_no_prune} describes a valid tree traversal $\zb$ for $\xb$ (i.e., $z_t=1$ for any $t \in \cT$ if $\xb$ reaches $t$).
\end{lemma}

\begin{proof}
Assume without loss of generality that $s_{\theta_{t}}(\xb) \neq 0$ for all $t \in \cT$ (if this does not hold for some node $\tilde{t}$ simply add a small constant $\epsilon$ to $s_{\theta_{\tilde{t}}}(\xb)$). Recall that for any non-leaf node $t'$ in the fixed binary tree $\cT$, a point $\xb$ is sent to the left child if $s_{\theta_{t'}}(\xb) < 0$ and to the right child otherwise. Following these splits from the root, every point $\xb$ reaches a unique leaf node $f_{\xb} \in \cF$, where $\cF$ is the set of all leaf nodes.
Notice first that both $\min L_{f_{\xb}} > 0$ and $ \min R_{f_{\xb}} > 0$.\footnote{We use the convention that $\min \varnothing = \infty$ (i.e., for paths that go down the rightmost or leftmost parts of the tree).}
This is because $\forall t' \in A_L(f_{\xb})$ it is the case that $s_{\theta_{t'}}(\xb) < 0$, and  $\forall t' \in A_R(f_{\xb})$ we have that $s_{\theta_{t'}}(\xb) > 0$.
This is due to the definition of $f_{\xb}$ as the leaf node that $\xb$ reaches (i.e., it traverses to $f_{\xb}$ by following the aforementioned splitting rules).
Further, for any other leaf $f \in \cF \setminus \{f_{\xb}\}$ we have that either $\min L_{f} < 0$ or $\min R_{f} < 0$. This is because there exists a node $t'$ on the path to $f$ (and so $t'$ is either in $A_L(f)$ or $A_R(f)$) that does not agree with the splitting rules. For this node it is either the case that (i) $s_{\theta_{t'}}(\xb) > 0$ and $t' \in A_L(f)$, or that (ii) $s_{\theta_{t'}}(\xb) < 0$ and $t' \in A_R(f)$. In the first case $\min L_f < 0$, in the second $\min R_f < 0$. Therefore we have that $q_f < 0$ for all $f \in \cF \setminus \{f_{\xb}\}$, and $q_{f_{\xb}} > 0$. In order to maximize the objective $\zb^\top \qb$ we will have that $z_{f_{\xb}} = 1$ and $z_f = 0$. Finally, let $\mathcal{N}_{f_{\xb}}$ be the set of non-leaf nodes visited by $\xb$ on its path to $f_{\xb}$. Now notice that the above argument applied to leaf nodes can be applied to nodes at any level of the tree: $q_t > 0$ for $ t \in \mathcal{N}_{f_{\xb}}$ while $q_{t'} < 0$ for $t' \in \cT \setminus \mathcal{N}_{f_{\xb}} \cup f_{\xb}$. Therefore $\zb_{\mathcal{N}_{f_{\xb}} \cup f_{\xb}} = 1$ and $\zb_{\cT \setminus \mathcal{N}_{f_{\xb}} \cup f_{\xb}} = 0$, completing the proof.
\end{proof}

This solution is unique so long as $s_{\theta_{t}}(\xb_i)
\!\neq\! 0$ for all $t \in \cT, i \in \{1, \ldots, n\}$ (\ie, $s_{\theta_{t}}(\xb_i) \!=\! 0$ means equal preference to split $\xb_i$ left or right). Further the integer constraint on $z_{it}$ can be relaxed to an interval constraint $z_{it} \in [0, 1]$ w.l.o.g. This is because if $s_{\theta_{t}}(\xb_i)
\neq 0$ then $z_t=0$ if and only if $q_t < 0$ and $z_t=1$ when $q_t > 0$ (and $q_t \neq 0$).


While the above program works for any fixed tree, we would like to be able to also learn the structure of the tree itself.
We do so by learning a pruning optimization variable $a_t \in \{0,1\}$, indicating if node $t \in \cT$ is active (if $1$) or pruned (if $0$). 
We adapt eq.~\eqref{eq:lp_no_prune} into the following pruning-aware mixed integer program (MIP) considering all inputs $\{\xb_i\}_{i=1}^n$:
\begin{align}
    \max_{\zb_1, \ldots, \zb_n, \ab} \;\;& \sum_{i=1}^n \zb_i^\top \qb_i - \frac{\lambda}{2} \|\ab\|^2 \label{eq:ourilp} \\ 
    \mbox{s.t. } \forall i \in [n],\;\;& a_t \leq a_{p(t)}, \;\;\; \forall t \in \cT \setminus \{ \text{root} \}\nonumber \\
    & z_{it} \leq a_t \nonumber \\
    & z_{it} \in [0, 1], a_t \in \{0,1\} \nonumber
\end{align}
with $\|\cdot\|$ denoting the $l_2$ norm.
We remove the first three constraints in eq.~\eqref{eq:lp_no_prune} as they are a deterministic computation independent of $\zb_1,\ldots,\zb_n,\ab$. We denote by $p(t)$ the parent of node $t$. The new constraints $a_t \leq a_{p(t)}$ ensure that child nodes $t$ are pruned if the parent node $p(t)$ is pruned, hence enforcing the tree hierarchy, while the other new constraint $z_{it} \leq a_t$ ensures that no point $\xb_i$ can reach node $t$ if node $t$ is pruned, resulting in losing the associated reward $q_{it}$. 
Overall, the problem consists in a trade-off, controlled by hyper-parameter $\lambda \in \mathbb{R}$, between minimizing the number of active nodes through the pruning regularization term (since $a_t \in \{0,1\}, \|\ab\|^2 = \sum_t a_t^2 = \sum_t \mathbb{I}(a_t = 1)$) while maximizing the reward from points traversing the nodes.

\input{sections/new_problem}

%% file: sections/new_problem.tex
\subsection{Learning Tree Parameters}

We now consider the problem of learning the splitting parameters $\theta_t$.
A natural approach would be to do so in the MIP itself, as in the optimal tree literature. However, this would severely restrict allowed splitting functions as even linear splitting functions can only practically run on at most thousands of training inputs \citep{bertsimas2017optimal}. Instead, we propose to learn $s_{\theta_{t}}$ via gradient descent.

To do so, we must be able to compute the partial derivatives
$\partial\ab/\partial\qb$ and $\partial\zb/\partial\qb$. However, the solutions of eq.~\eqref{eq:ourilp} are piecewise-constant, leading to null gradients.
To avoid this, we relax the integer constraint on $\ab$ to the interval $[0,1]$ and add quadratic regularization $-\nicefrac{1}{4}\sum_i (\| \zb_i \|^2 +\| 1 - \zb_i \|^2)$.
The regularization term for $\zb$ is symmetric so that it shrinks solutions toward even left-right splits
(see Figure~\ref{fig:split-z} for a visualization, and the supplementary for a more detailed justification with an empirical evaluation of the approximation gap). Rearranging and negating the objective yields
\begin{figure}
\centering
\input{sections/fig_z.tex}
\caption{\label{fig:split-z}Routing of point $\xb_i$ at node $t$, without pruning, without (left) or with (right) quadratic relaxation. Depending on the decision split $s_{\theta_t}(\cdot)$, $\xb_i$ reaches node $t$'s left child $l$ (right child $r$ respectively) if $z_{il} > 0$ ($z_{ir} > 0$).
The relaxation makes $\zb_i$ continuous and encourages points to have a margin ($| s_{\theta_t}| > 0.5$).}
\end{figure}
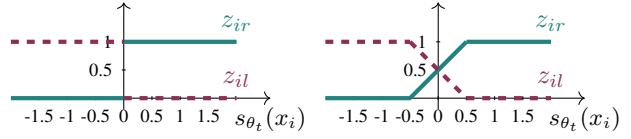

\begin{align}
   \mathrm{T}_{\lambda}(\qb_1, \ldots, \qb_n) =  \label{eq:qp_prune} \\ \argmin_{\zb_1, \ldots, \zb_n, \ab} \;\;& \nicefrac{\lambda}{2} \| \ab \|^2 + \hlf\sum_{i=1}^n \| \zb_i - \qb_i -\nicefrac{1}{2} \|^2 \nonumber
 \\
    \mbox{s.t. } \forall i \in [n],\;\;& a_t \leq a_{p(t)}, \;\;\; \forall t \in \cT
\setminus \{\text{root}\}  \nonumber \\
    & z_{it} \leq a_t \nonumber \\
    & z_{it} \in [0, 1], a_t \in [0,1]. \nonumber
\end{align}

The regularization makes the objective strongly convex. It follows from convex duality that $\mathrm{T}_{\lambda}$ is Lipschitz continuous
\citep[Corollary 3.5.11]{zalinescu}.
By Rademacher's theorem \citep[Theorem 9.1.2]{borweinlewis}, $\mathrm{T}_{\lambda}$ is thus differentiable almost everywhere.
Generic methods such as OptNet~\citep{optnet} could be used to compute the
solution and gradients. However,
by using the tree structure of the constraints, we next derive an efficient specialized algorithm.
The main insight, shown below, reframes pruned binary tree learning
as isotonic optimization.

\begin{proposition}\label{prop:iso}
Let $\mathcal{C}= \big\{ \ab \in \mathbb{R}^{|\mathcal{T}|} :
a_t \leq a_{p(t)} ~~\text{for all}~~ t \in \mathcal{T} \setminus \{ \text{root}
\} \big\}.$
Consider $\ab^\star$ the optimal value of
\begin{equation}\label{eq:nested_qp_prune}
\argmin_{\ab\in \mathcal{C} \cap [0,1]^{|\mathcal{T}|}}
\sum_{t \in \mathcal{T}} \Big( \frac{\lambda}{2} a_t^2
+\!\!\sum_{i: a_t < q_{it} + \hlf}\!\!{(a_t - q_{it} - \hlf)}^2 \Big)\,.
\end{equation}
Define%
\footnote{Here $\operatorname{Proj}_{\mathcal{S}}(x)$ is the projection of $x$ onto set $\mathcal{S}$. If $\mathcal{S}$ are box constraints, projection amounts to clipping.}
$\zb_i^\star$ such that $z^\star_{it} = \operatorname{Proj}_{[0, a_t^\star]}(q_{it} + \hlf)$.
Then, $\mathrm{T}_{\lambda}(\qb_1, \ldots, \qb_n) = \zb^\star_1, \ldots, \zb^\star_n, \ab^\star$.
\end{proposition}
\begin{proof}
The constraints and objective of eq.~\eqref{eq:qp_prune} are separable,
so we may push the minimization \wrt $\zb$ inside the objective:
\begin{equation}
\argmin_{\ab \in \mathcal{C} \cap [0,1]^{|\mathcal{T}|}} \! \frac{\lambda}{2} \| \ab \|^2
\!\!+\!\!\sum_{t \in \mathcal{T}}\!\sum_{i=1}^n
\min_{0 \leq z_{it} \leq a_t} \hlf{({z}_{it} - {q}_{it} - \hlf)}^2\,.
\end{equation}
Each inner minimization
$\min_{0 \leq z_{it} \leq a_t} \hlf (z_{it} - q_{it} - \hlf)^2$
is a one-dimensional projection onto box constraints, with solution
$z^\star_{it} = \operatorname{Proj}_{[0, a_t]}(q_{it} + \hlf)$.
We use this to eliminate $\zb$,
noting that each term $\hlf(z^\star_{it} - q_{it} - \hlf)^2$ yields
\begin{equation}
\begin{cases}
\hlf (q_{it} + \hlf)^2, & q_{it} < - \hlf\\
0, & - \hlf \leq q_{it} < a_t - \hlf \\
\hlf(a_t - q_{it} - \hlf)^2, & q_{it} \geq a_t - \hlf.\\
\end{cases}
\end{equation}
The first two branches are constants \wrt $a_t$
The problems in eq.~\eqref{eq:qp_prune} and eq.~\eqref{eq:nested_qp_prune}
differ by a constant and have the same constraints,
therefore they have the same solutions.
\end{proof}

\begin{figure}
    \centering
    \includegraphics[width=.42\textwidth]{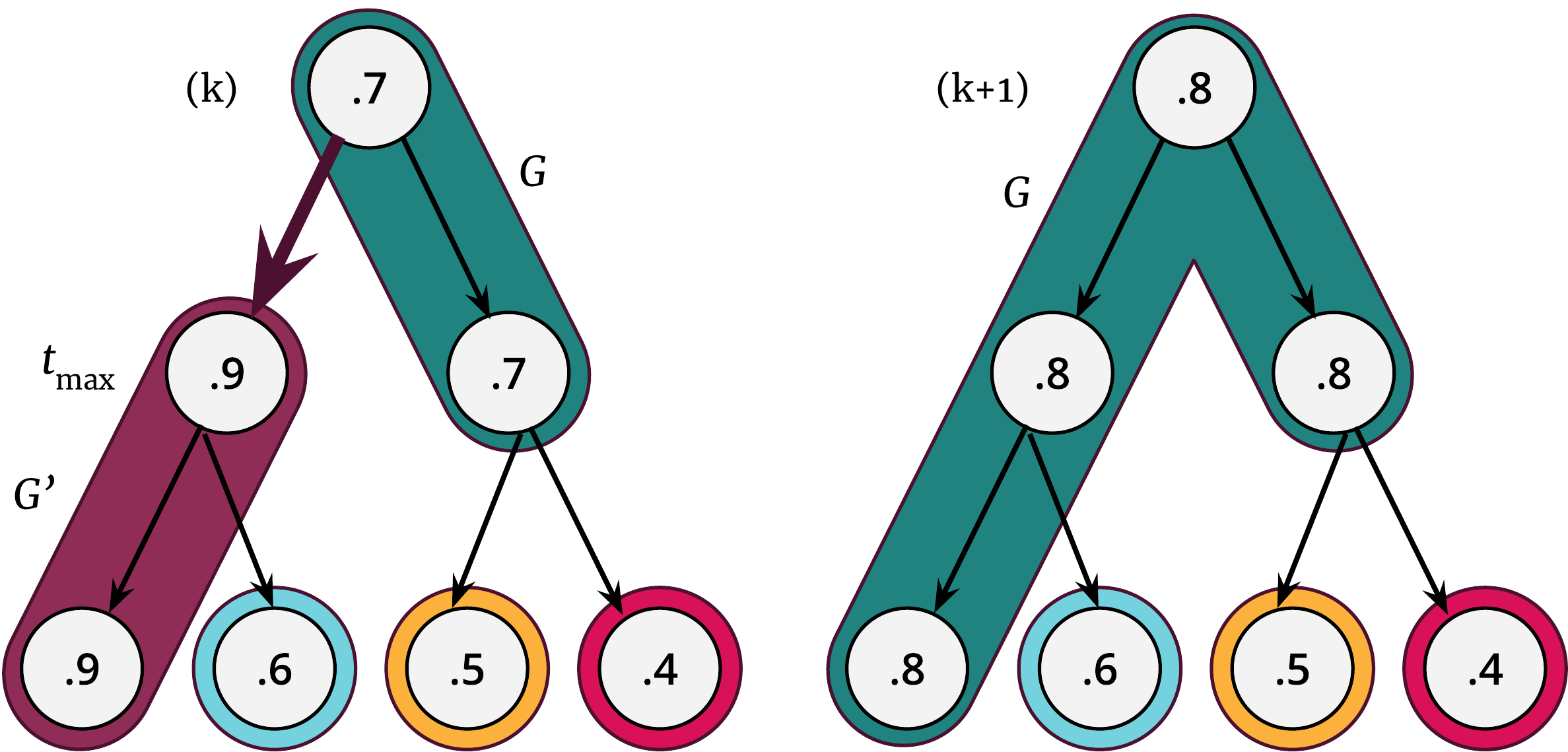}
    \caption{One step in the tree isotonic optimization process from Algorithm~\ref{alg}.
    Shaded tubes depict the partition groups $\mathcal{G}$, and the emphasized
arc is the largest among the violated inequalities. The group of $t_{\max}$
    merges with the group of its parent,
    and the new value for the new, larger group $a_G$ is computed.}
    \label{fig:coloring}
\end{figure}

\begin{algorithm}[t]\small
\caption{Pruning via isotonic optimization\label{alg}}
\begin{algorithmic}
\STATE initial partition
$\mathcal{G} \leftarrow \big\{ \{1\}, \{2\}, \cdots \}
\subset 2^\mathcal{T}$
\FOR {$G \in \mathcal{G}$}
\STATE $a_G \leftarrow \argmin_a \sum_{t \in G} g_t(a)$
\hfill\COMMENT{Eq.~\eqref{eq:nested_qp_prune_isotonic}, Prop.~\ref{prop:scalar_subproblem}}
\ENDFOR
\REPEAT
\STATE $t_\text{max} \leftarrow \argmax_t \{a_t: a_t > a_{p(t)}\}$
\STATE merge $G \leftarrow G \cup G'$ where $G' \ni t_{\max}$ and $G \ni p(t_{\max})$.
\STATE update  $a_G \leftarrow \argmin_a \sum_{t \in G} g_t(a)$
\hfill\COMMENT{Eq.~\eqref{eq:nested_qp_prune_isotonic}, Prop.~\ref{prop:scalar_subproblem}}
\UNTIL{ no more violations $a_t > a_{p(t)}$ exist}
\end{algorithmic}
\end{algorithm}

\paragraph{Efficiently inducing trees as isotonic optimization.}

From Proposition~\ref{prop:iso},
notice that eq.~\eqref{eq:nested_qp_prune} is an instance of tree-structured isotonic
optimization:
the objective decomposes over \emph{nodes}, and the inequality constraints
correspond to edges in a rooted tree:
\begin{align}\label{eq:nested_qp_prune_isotonic}
&\argmin_{\ab \in \mathcal{C}} \sum_{t \in \mathcal{T}} g_t(a_t)\,, ~~~~\text{where} \\
 g_t(a_t) =&
\frac{\lambda}{2} a_t^2
+\!\sum_{i: a_t \leq q_{it} + \hlf}\!\!\hlf {(a_t - q_{it} - \hlf)}^2  + \iota_{[0,1]}(a_t). \nonumber
\end{align}
where $\iota_{[0,1]}(a_t)\!=\!\infty$ if $a_t \notin [0,1]$ and $0$ otherwise.
This problem can be solved by a generalized \emph{pool adjacent violators} (PAV)
algorithm: Obtain a tentative solution by ignoring the constraints, then
iteratively remove violating edges $a_t > a_{p(t)}$ by \emph{pooling together}
the nodes at the end points.
Figure~\ref{fig:coloring} provides an illustrative example of the procedure.
At the optimum, the nodes are organized into a
partition $\mathcal{G} \subset 2^\mathcal{T}$, such that if two nodes $t, t'$
are in the same group $G \in \mathcal{G}$, then $a_t = a_{t'} \coloneqq a_G$.

When the inequality constraints are the edges of a rooted tree, as is the case
here, the PAV algorithm finds the optimal solution in at most $|\mathcal{T}|$ steps,
where each involves updating the $a_G$ value for a newly-pooled group by solving
a one-dimensional subproblem of the form \citep{yu2016exact}\footnote{Compared to \citet{yu2016exact}, our tree inequalities are in the opposite
direction.  This is equivalent to a sign flip of parameter $\ab$, \ie,
to selecting the \emph{maximum} violator rather than the minimum one at each
iteration.}
\begin{equation}\label{eq:scalar_subproblem}
a_G = \argmin_{a \in \bR} \sum_{t \in G} g_t(a)\,,
\end{equation}
resulting in Algorithm~\ref{alg}.
It remains to show how to solve eq.~\eqref{eq:scalar_subproblem}. The next
result, proved in the supplementary, gives an exact and efficient solution,
with an algorithm that requires finding the nodes with highest
$q_{it}$ (\ie, the nodes which $\xb_i$ is most highly encouraged to traverse).

\begin{figure}
    \centering
    {\includegraphics[width=0.4\textwidth]{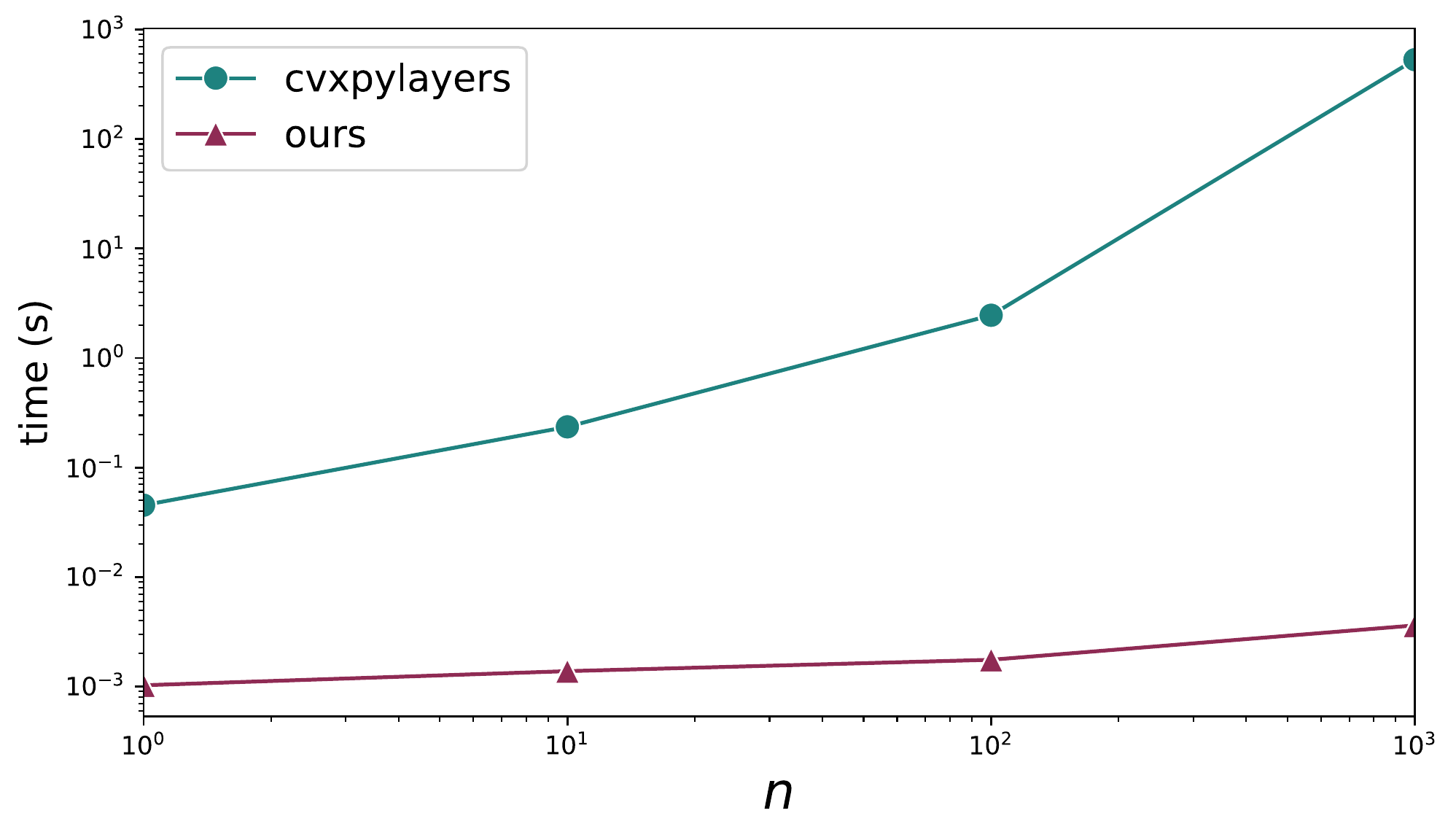}
    }
    \caption{Comparison of Algorithm~\ref{alg}'s running time (ours) with the running time of cvxpylayers~\citep{cvxpylayers2019} for tree depth $D \!=\! 3$, varying $n$.}
    \label{fig:time}
        \label{fig:lambda}
\end{figure}

\pagebreak
\begin{proposition}\label{prop:scalar_subproblem}%
The solution to the one-dimensional problem in
eq.~\eqref{eq:scalar_subproblem} for any $G$ is given by
\begin{align}\argmin_{a \in \bR} \sum_{t \in G} g_t(a) =
\operatorname{Proj}_{[0,1]}\big(a(k^\star)\big) \\ \text{where}\quad a(k^\star) \coloneqq
\frac{\sum_{(i,t) \in S(k^\star)} (q_{it} + \hlf)}{\lambda|G|+k^\star}, \nonumber
\end{align}
$S(k) = \{j^{(1)}, \ldots, j^{(k)}\}$ is the set of indices $j=(i,t) \in
\{1, \ldots, n\} \times G$ into
the $k$ highest values of $\qb$, \ie,
$q_{j^{(1)}} \geq
q_{j^{(2)}} \geq \ldots \geq
q_{j^{(m)}}$,
and $k^\star$ is the smallest $k$ satisfying
$a(k) > q_{j^{(k+1)}} + \hlf$.
\end{proposition}

Figure~\ref{fig:time} compares our algorithm to a generic differentiable solver (details and additional comparisons in supplementary material).

\paragraph{Backward pass and efficient implementation details.}
Algorithm~\ref{alg} is a sequence of differentiable operations that can be
implemented \emph{as is} in automatic differentiation frameworks. However, because of the prominent loops and indexing operations, we opt for a low-level implementation as a \texttt{C++} extension.
Since the $\qb$ values are constant \wrt $\ab$, we only need to sort them
once as preprocessing, resulting in a overall time complexity of $\mathcal{O}(n |\cT| \log (n |\cT|))$ and space complexity of $\mathcal{O}(n |\cT|)$.
For the backward pass, rather than relying on automatic
differentiation, we make two remarks about the form of $\ab$.
Firstly, its elements are organized in groups, \ie, $a_t = a_t' = a_G$ for $\{t,
t'\} \subset G$. Secondly, the value $a_G$ inside each group
depends only on  the optimal support set $S^\star_G \coloneqq S(k^\star)$ as
defined for each subproblem  by Proposition~\ref{prop:scalar_subproblem}.
Therefore, in the forward pass, we must store only the node-to-group mappings
and the sets $S^\star_G$. Then, if $G$ contains $t$,
\begin{equation*}
\pfrac{a^\star_t}{q_{it'}} = \begin{cases}
\frac{1}{\lambda |G| + k^\star},
& 0 < a^\star_t < 1,
~\text{and}~
(i,t') \in S^\star_G,\\
0,&\text{otherwise.}
\end{cases}
\end{equation*}
As $\mathrm{T}_{\lambda}$ is differentiable almost everywhere, these
expressions yield the unique Jacobian at all but a measure-zero set of
points, where they yield one of the Clarke generalized Jacobians
\citep{clarke_book}.
We then rely on automatic differentiation to propagate gradients from $\zb$ to $\qb$, and from $\qb$ to
the split parameters $\bm{\theta}$: since $\qb$ is defined elementwise via $\min$
functions, the gradient propagates through the minimizing path,
by Danskin's theorem \citep[Proposition B.25,][]{bertsekas-nonlin,danskin_theorem}.
 \subsection{The Overall Objective}

\begin{algorithm}[t]\small
\caption{\label{alg:learn}Learning with decision tree representations.}
\begin{algorithmic}
\STATE initialize neural network parameters $\bm{\phi},\bm{\theta}$
\REPEAT
\STATE sample batch $\{\xb_i\}_{i \in B}$
\STATE induce traversals:
\STATE $\{\zb_i\}_{i \in B}, \ab =
\mathrm{T}_\lambda\big(\{q_{\bm{\theta}}(\xb_i)\}\big)$
\hfill\COMMENT{Algorithm~\ref{alg}; differentiable}
\STATE update parameters using
$\nabla_{\{\bm{\theta},\bm{\phi}\}}
\ell(f_{\bm{\phi}} (\xb_i, \zb_i))$
\hfill\COMMENT{autograd}
\UNTIL{convergence or stopping criterion is met}
\end{algorithmic}
\end{algorithm}


We are now able to describe the overall optimization procedure that simultaneously learns tree parameters: (a) input traversals $\zb_1, \ldots, \zb_n$; (b) tree pruning $\ab$; and (c) split parameters $\bm{\theta}$.
Instead of making predictions using heuristic scores over the training points assigned to a leaf (e.g., majority class), we learn a prediction function $f_\phi(\zb, \xb)$ that minimizes an arbitrary loss $\ell(\cdot)$ as follows:
\begin{equation}
\begin{aligned}
    \min_{\bm{\theta},\bm{\phi}}\;\;& \sum_{i=1}^n \ell\big(f_\phi(\xb_i, \zb_i)\big) \label{eq:obj} \\
    \mbox{where}\;\;& \zb_1, \ldots, \zb_n, \ab := \mathrm{T}_{\lambda}\big(q_{\bm{\theta}}(\xb_1), \ldots, q_{\bm{\theta}}(\xb_n)\big)\,.
\end{aligned}
\end{equation}
This corresponds to embedding a decision tree as a layer of a deep neural network (using $\zb$ as intermediate representation) and optimizing the parameters $\bm{\theta}$ and $\bm{\phi}$ of the whole model by back-propagation.
In practice, we perform mini-batch updates for efficient training;
the procedure is sketched in Algorithm~\ref{alg:learn}.
Here we define $q_{\theta}(\xb_i) := \qb_i$ to make explicit the dependence of $\qb_i$ on $\bm{\theta}$.

%% file: sections/fig_z.tex
\tikzstyle{zir} = [CGreen, ultra thick]
\tikzstyle{zil} = [CMagenta, ultra thick, dashed]
\begin{tikzpicture}[scale = 0.75]
  \draw[->] (-2, 0) -- (2.6, 0) node[below] {\small $s_{\theta_t}(x_i)$};
  \draw[->] (0, -0.2) -- (0, 1.7);
  \foreach \x in {-1.5,-1,...,1.5}
 	\draw (\x,1pt) -- (\x,-1pt)
	node[anchor=north] {\scriptsize \x};
  \foreach \y in {0.5,1}
 	\draw (1pt,\y) -- (-1pt,\y) 
 		node[anchor=east] {\scriptsize \y}; 
 \draw[domain=-2:0, variable=\x,zir] plot ({\x}, {0});
  \draw[domain=0:2, variable=\x,zir] plot ({\x}, {1}) node[above] {$z_{ir}$};
  \draw[domain=-2:0, variable=\x,zil] plot ({\x}, {1});
  \draw[domain=0:2, variable=\x, zil] plot ({\x}, {0}) node[above] {$z_{il}$};
\end{tikzpicture}
\hfill
\begin{tikzpicture}[scale = 0.75]
  \draw[->] (-2, 0) -- (2.6, 0) node[below] {\small $s_{\theta_t}(x_i)$};
  \draw[->] (0, -0.2) -- (0, 1.7);
  \foreach \x in {-1.5,-1,...,1.5}
 	\draw (\x,1pt) -- (\x,-1pt)
	node[anchor=north] {\scriptsize \x};
  \foreach \y in {0.5,1}
 	\draw (1pt,\y) -- (-1pt,\y) 
 		node[anchor=east] {\scriptsize \y}; 
 \draw[domain=-2:-0.5, variable=\x, zir] plot ({\x}, {0});
 \draw[domain=-0.5:0.5, variable=\x, zir] plot ({\x}, {\x + 0.5});
  \draw[domain=0.5:2, variable=\x, zir] plot ({\x}, {1}) node[above] {$z_{ir}$};
  \draw[domain=-2:-0.5, variable=\x, zil] plot ({\x}, {1});
 \draw[domain=-0.5:0.5, variable=\x, zil] plot ({\x}, {-\x + 0.5});
  \draw[domain=0.5:2, variable=\x, zil] plot ({\x}, {0}) node[above] {$z_{il}$};
\end{tikzpicture}

%% file: sections/experiments.tex
In this section we showcase our method on both: (a) \emph{Classification/Regression} for tabular data, where tree-based models have been demonstrated to have superior performance over MLPs \citep{popov2019neural}; and (b) \emph{Hierarchical clustering} on unsupervised data. Our experiments demonstrate that our method leads to predictors that are competitive with state-of-the-art tree-based approaches, scaling better with the size of datasets and generalizing to many tasks.
Further we visualize the trees learned by our method and how sparsity is easily adjusted by tuning the hyper-parameter $\lambda$.

\paragraph{Architecture details.}
We use a linear function or a multi-layer perceptron ($L$ fully-connected layers
with ELU activation~\citep{clevert2015fast} and dropout) for $f_\phi(\cdot)$ and
choose between linear or linear followed by ELU splitting functions $s_\theta(\cdot)$ (we limit
the search for simplicity, there are no restrictions
except differentiability).





\subsection{Supervised Learning on Tabular Datasets}
\label{sec:tabular}
Our first set of experiments is on supervised learning with heterogeneous tabular datasets, where we consider both regression and binary classification tasks.
We minimize the Mean Square Error (MSE) on regression datasets and the Binary Cross-Entropy (BCE) on classification datasets.
We compare our results with tree-based architectures, which either train a single or an ensemble of decision trees.
Namely, we compare against the greedy CART algorithm  \citep{breiman1984classification} and two optimal decision tree learners:
OPTREE with local search~\citep[Optree-LS, ][]{dunn2018optimal} 
and a state-of-the-art optimal tree method \citep[GOSDT, ][]{lin2020generalized}.
We also consider three baselines with probabilistic routing:
deep neural decision trees \citep[DNDT, ][]{yang2018deep},
deep neural decision forests~\citep{kontschieder2015deep} configured to jointly optimize the routing and the splits and to use an ensemble size of 1 (NDF-1), and adaptive neural networks~\citep[ANT, ][]{tanno2019adaptive}.
As for the ensemble baselines, we compare to NODE~\citep{popov2019neural}, the state-of-the-art method for training a forest of differentiable oblivious decision trees on tabular data, and to XGBoost~\citep{Chen:2016:XST:2939672.2939785}, a scalable tree boosting method.
%
We carry out the experiments on the following datasets. 
\textbf{Regression:}
{\em Year} \citep{msd},
Temporal regression task constructed from the Million Song Dataset;
{\em Microsoft} \citep{mslr},
Regression approach to the MSLR-Web10k Query--URL relevance prediction for
learning to rank;
{\em Yahoo} \citep{chapelle2011yahoo}, Regression approach to the C14 learning-to-rank challenge.
\textbf{Binary classification:}
{\em Click}, Link click prediction based on the KDD Cup 2012
dataset, encoded and subsampled following \citet{popov2019neural};
{\em Higgs} \citep{higgs}, prediction of Higgs boson--producing
events.

\begin{table}[t!]
\caption{\label{tab:tabular-reg}%
Results on tabular regression datasets. We report average and standard deviations over $4$ runs of MSE and \textbf{bold} best results, and those within a standard deviation from it, for each family of algorithms (single tree or ensemble).
For the single tree methods we additionally report the average training times (s).}

\begin{center}
\scalebox{0.85}{
\begin{tabular}{clccc}
\hline
 &METHOD  & YEAR & MICROSOFT & YAHOO \\ \hline
\parbox[t]{2mm}{\multirow{3}{*}{\rotatebox[origin=c]{90}{Single}}}
& CART & $96.0 \pm 0.4$ & $0.599 \pm 1\text{e-}3$ & $0.678 \pm 17\text{e-}3$\\
& ANT   & $\bm{77.8 \pm 0.6}$ & $\bm{0.572 \pm 2\text{e-}3}$ & $\bm{0.589 \pm 2\text{e-}3}$\\
& Ours    &$\bm{77.9\pm0.4}$ & $\bm{0.573\pm2\text{e-}3}$ & $\bm{0.591 \pm 1\text{e-}3}$  \\\hline
\parbox[t]{2mm}{\multirow{2}{*}{\rotatebox[origin=c]{90}{Ens.}}}
& NODE     &$\bm{76.2 \pm 0.1}$ & $0.557 \pm 1\text{e-}3$ & $0.569\pm1\text{e-}3$ \\
& XGBoost  &$78.5 \pm 0.1$ & $\bm{0.554 \pm 1\text{e-}3}$ & $\bm{0.542\pm1\text{e-}3}$\\\hline\hline
\parbox[t]{2mm}{\multirow{3}{*}{\rotatebox[origin=c]{90}{Time}}}
& CART & $23$s & $20$s & $26$s \\
& ANT & $4674$s & $1457$s & $4155$s \\
& ours & $1354$s & $1117$s & $825$s \\ \hline
\end{tabular}}
\end{center}
\end{table}

For all tasks, we follow the preprocessing and task setup from
\citep{popov2019neural}.
All datasets come with training/test splits.
We make use of 20\% of the training set as validation set for selecting the best model over training and for tuning the hyperparameters.
We tune the hyperparameters for all methods.
and optimize eq.~\eqref{eq:obj} and all neural network methods (DNDT, NDF, ANT and NODE) using the Quasi-Hyperbolic Adam~\citep{ma2018quasi} stochastic gradient descent method.
Further details are provided in the supplementary.
Tables~\ref{tab:tabular-reg} and~\ref{tab:tabular-cls} report the obtained results on the regression and classification datasets respectively.\footnote{GOSDT/DNDT/Optree-LS/NDF are for classification only.}
Unsurprisingly, ensemble methods outperfom single-tree ones on all datasets, although at the cost of being harder to visualize/interpret.
Our method has the advantage of (a) generalizing to any task; (b) outperforming or matching all single-tree methods; (c) approaching the performance of ensemble-based methods; (d) scaling well with the size of datasets.
These experiments show that our model is also significantly faster to train, compared to its differentiable tree counterparts NDF-1 and ANT, while matching or beating the performance of these baselines, and it generally provides the best trade-off between time complexity and accuracy over all datasets (visualizations of this trade-off are reported in the supplementary material).
Further results on smaller datasets are available in the supplementary material to provide a comparison with optimal tree baselines.

\begin{table}
\caption{\label{tab:tabular-cls}%
Results on tabular classification datasets. We report average and standard deviations over $4$ runs of error rate. Best result for each family of algorithms (single tree or ensemble) are in \textbf{bold}.  Experiments are run on a machine with 16 CPUs and 64GB of RAM, with a training time limit of $3$ days. We denote methods that exceed this memory and training time as OOM and OOT, respectively.
For the single tree methods we additionally report the average training times (s) when available.}

\begin{center}
\scalebox{0.9}{
\begin{tabular}{clcc}
\hline
 &METHOD  & CLICK & HIGGS \\ \hline
\parbox[t]{2mm}{\multirow{7}{*}{\rotatebox[origin=c]{90}{Single Tree}}}
& GOSDT   & OOM & OOM \\
& OPTREE-LS  & OOT & OOT\\
& DNDT    & $0.4866 \pm 1\text{e-}2$ & OOM\\
& NDF-1   & $0.3344 \pm 5\text{e-}4$ & $0.2644 \pm 8\text{e-}4$\\
& CART & $0.3426 \pm 11\text{e-}3$ & $0.3430\pm 8\text{e-}3$\\
& ANT     & $0.4035 \pm 0.1150$ & $0.2430 \pm 6\text{e-}3$\\
& Ours    & $\bm{0.3340 \pm 3\text{e-}4}$ & $\bm{0.2201\pm3\text{e-}4}$ \\ \hline
\parbox[t]{2mm}{\multirow{2}{*}{\rotatebox[origin=c]{90}{Ens.}}}
& NODE     & $\bm{0.3312 \pm 2\text{e-}3}$ & $\bm{0.210\pm5\text{e-}4}$ \\
& XGBoost  & $\bm{0.3310 \pm 2\text{e-}3}$ & $0.2334 \pm 1\text{e-}3$\\\hline\hline
\parbox[t]{2mm}{\multirow{5}{*}{\rotatebox[origin=c]{90}{Time}}}
& DNDT & $681$s & -\\
& NDF-1 & $3768$s & $43593$s \\
& CART & $3$s & $113$s\\
& ANT & $75600$s & $62335$s\\
& ours & $524$s & $18642$s \\ \hline
\end{tabular}}
\end{center}

\end{table}

\begin{figure*}

    \centering
    \includegraphics[width=\textwidth]{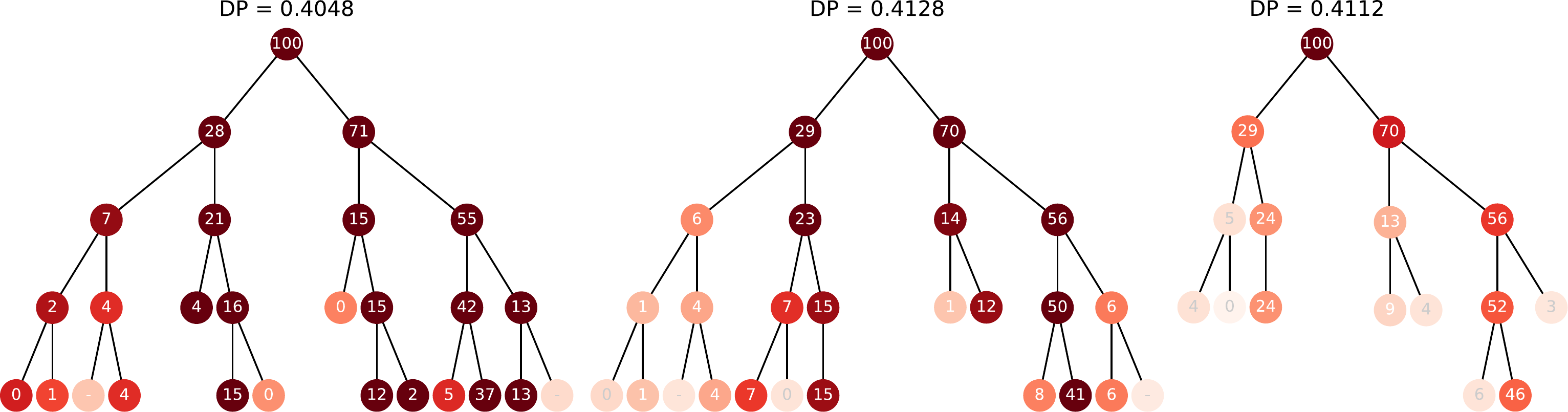}
    \caption{\label{fig:pruning}%
{\em Glass} tree routing distribution, in rounded percent of dataset, for $\lambda$ left-to-right in $\{1, 10, 100\}$.
The larger $\lambda$, the more nodes are pruned.
We report dendrogram purity (DP) and represent the nodes by the percentage of
points traversing them (normalized at each depth level) and with a color
intensity depending on their $\ab_t$ (the darker the higher $\ab_t$). The empty nodes are labeled by a dash $-$ and the inactive nodes ($a_t \!=\! 0$) have been removed.
}
\end{figure*}

\subsection{Self-Supervised Hierarchical Clustering}\label{sec:clust}

\begin{table}
    \caption{\label{tab:clust}%
Results for hierarchical clustering. We report average and standard deviations of {\em dendrogram purity} over four runs.} 
    \begin{center}\small
        \begin{tabular}{lcc}\hline
        METHOD  & GLASS & COVTYPE \\
        \hline
        Ours         & $0.468\pm0.028$ & $\mathbf{0.459\pm0.008}$ \\
        gHHC             &$0.463\pm0.002$ & $0.444\pm0.005$ \\
        HKMeans             &$\mathbf{0.508\pm0.008}$ & $0.440\pm0.001$ \\
        BIRCH             &$0.429\pm0.013$ & $0.440\pm0.002$ \\\hline
        \end{tabular}
    \end{center}
\end{table}

To show the versatility of our method, we carry out a second set of experiments on hierarchical clustering tasks. Inspired by the recent success of self-supervised learning approaches \citep{lan2019albert,he2020momentum}, we learn a tree for hierarchical clustering in a self-supervised way. Specifically, we regress a subset of input features from the remaining features, minimizing the MSE. This allows us to use eq.~\eqref{eq:obj} to learn a hierarchy (tree).
To evaluate the quality of the learned trees, we compute their dendrogram purity \citep[DP;][]{monath2019gradient}.
DP measures the ability of the learned tree to separate points from different classes, and corresponds to the expected purity of the least common ancestors of points of the same class.

We experiment on the following datasets: {\em Glass} \citep{mldata}:
glass identification for forensics, and {\em Covtype}
\citep{blackard1999comparative,mldata}: cartographic variables for forest cover
type identification.
For {\em Glass}, we regress features `Refractive Index' and `Sodium,' and for
{\em Covtype} the horizontal and vertical `Distance To Hydrology.'
We split the datasets into training/validation/test sets, with sizes 60\%/20\%/20\%. Here we only consider linear $f_\phi$.
As before, we optimize Problem~\ref{eq:obj} using the Quasi-Hyperbolic Adam
algorithm and tune the hyper-parameters using the validation reconstruction error.

As baselines, we consider:
BIRCH~\citep{zhang1997birch} and Hierarchical KMeans (HKMeans), the standard methods for performing clustering on large datasets; and the recently proposed gradient-based Hyperbolic Hierarchical Clustering \citep[gHHC, ][]{monath2019gradient} designed to construct  trees in hyperbolic space.
Table~\ref{tab:clust} reports the dendrogram purity scores for all methods. Our method yields results comparable to all baselines, even though not specifically tailored to hierarchical clustering.

\paragraph{Tree Pruning}
The hyper-parameter $\lambda$ in Problem~\ref{eq:qp_prune} controls how aggressively the tree is pruned, hence the amount of tree splits that are actually used to make decisions.
This is a fundamental feature of our framework as it allows to smoothly trim the portions of the tree that are not necessary for the downstream task, resulting in lower computing and memory demands at inference. 
In Figure~\ref{fig:pruning}, we study the effects of pruning on the tree learned on {\em Glass} with a depth fixed to $D \!=\! 4$.
We report how inputs are distributed over the learned tree for different values of $\lambda$.
We notice that increasing $\lambda$ effectively prune nodes and entire portions of the tree, without significantly impact performance (as measured by dendrogram purity).

To look into the evolution of pruning during training, we further plot the \% of active (unpruned) nodes within a training epoch in Figure~\ref{fig:epoch-an}.
We observe that (a) it appears possible to increase and decrease this fraction through training (b) the fraction seems to stabilize in the range $45\%$-$55\%$ after a few epochs.

\begin{figure}
    \centering
    \includegraphics[width=0.3\textwidth]{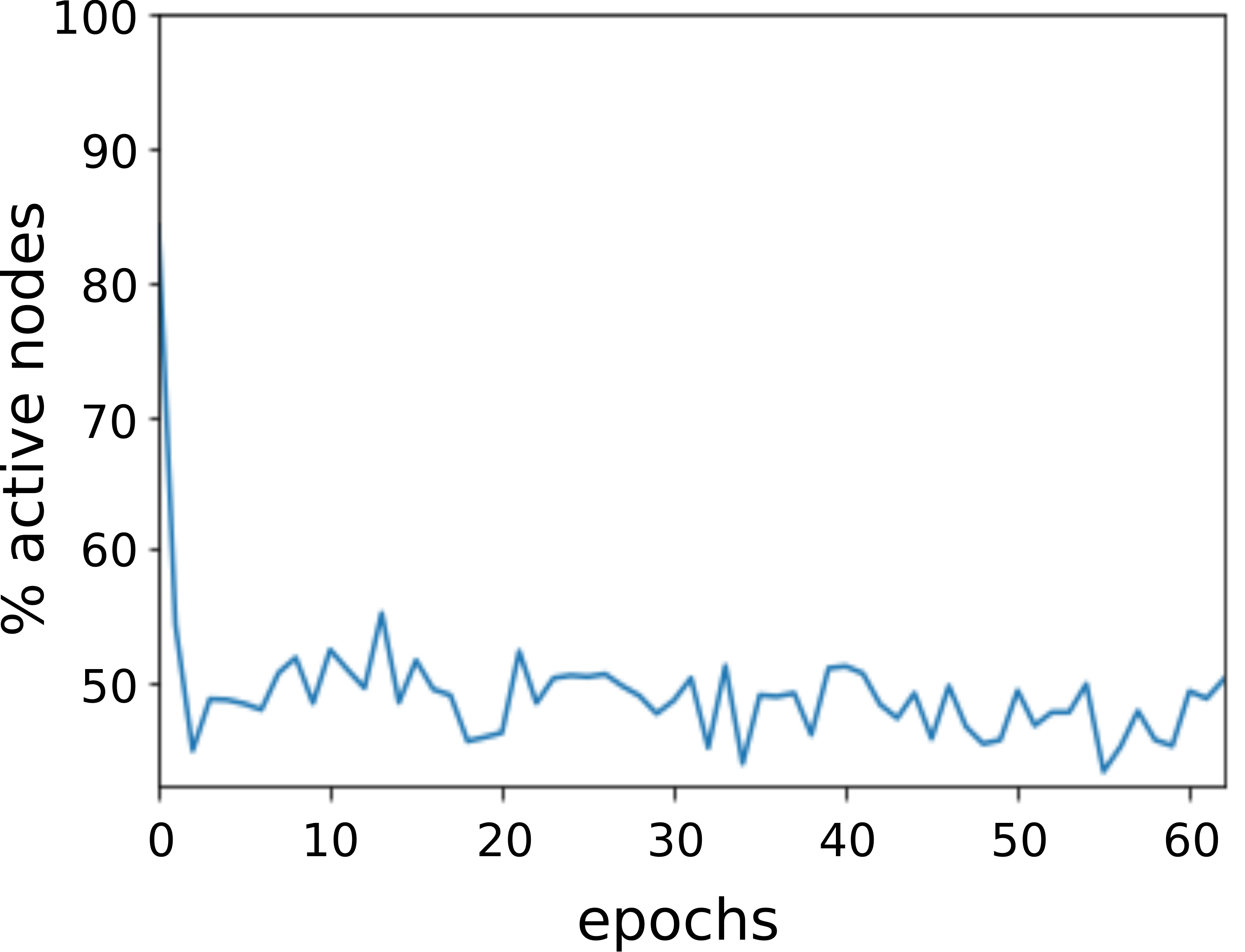}
    \caption{Percentage of active nodes during training as a function of the number of epochs on {\em Glass}, $D \!=\! 6, \lambda\!=\!10$.}
    \label{fig:epoch-an}
\end{figure}

\paragraph{Class Routing}
To gain insights on the latent structure learned by our method, we study how points are routed through the tree, depending on their class.
The {\em Glass} dataset is particularly interesting to analyze as its classes come with an intrinsic hierarchy, \eg, with superclasses \textit{Window} and \textit{NonWindow}.
Figure~\ref{fig:class-routing} reports the class routes for four classes.
As the trees are constructed without supervision, we do not expect the structure to exactly reflect the class partition and hierarchy. Still, we observe that points from the same class or super-class traverse the tree in a similar way.
Indeed, trees for class \textit{Build}~\ref{fig:build}
and class \textit{Vehicle}~\ref{fig:vehicle} that both belong to the \textit{Window} super-class, share similar paths, unlike the classes Containers~\ref{fig:cont}
and Headlamps~\ref{fig:head}.

\begin{figure}[t!]
    \centering
    \subfigure[Window-Float-Build]{\label{fig:build}\includegraphics[width=0.24\textwidth]{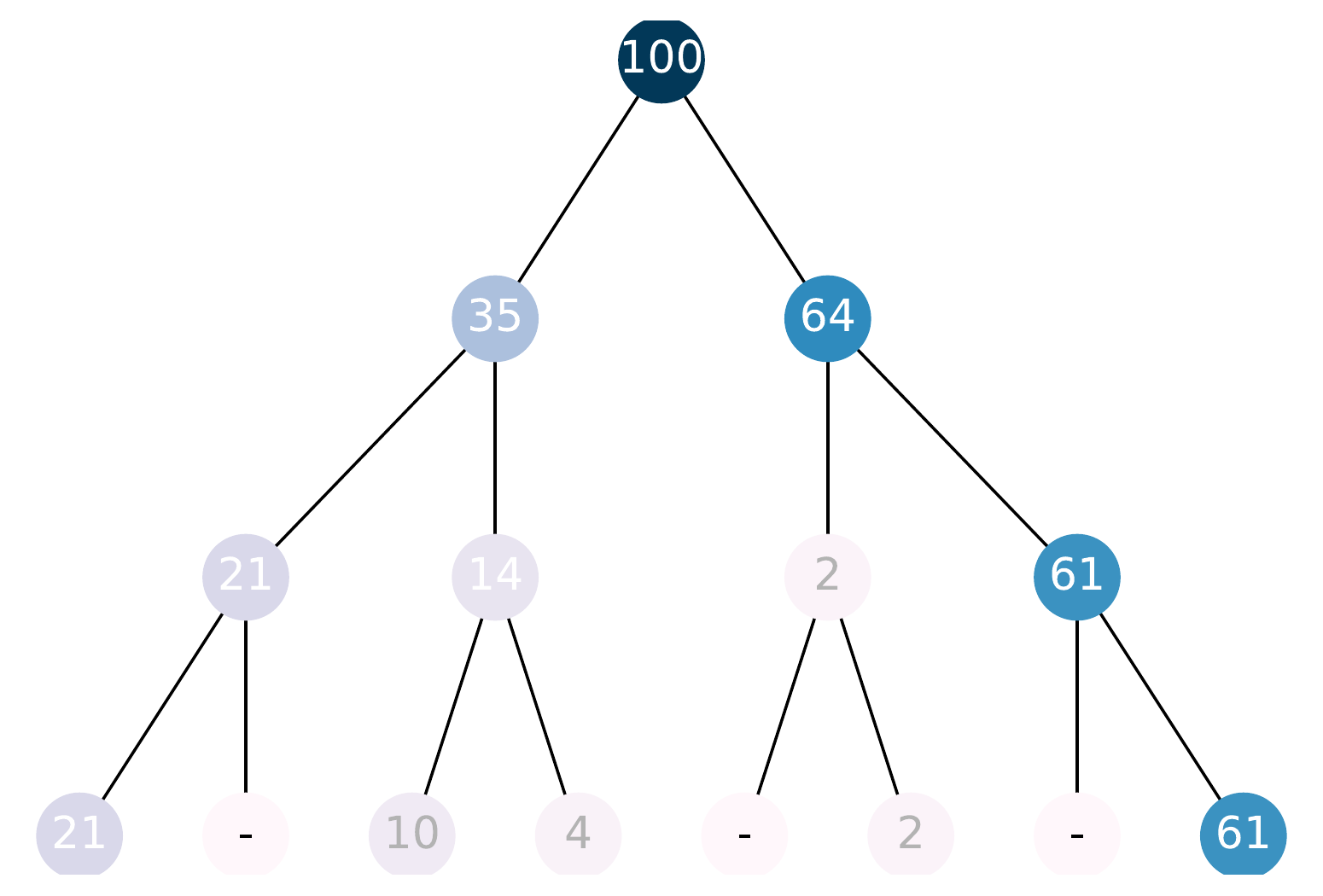}}\hfill
    \subfigure[Window-Float-Vehicle]{\label{fig:vehicle}\includegraphics[width=0.24\textwidth]{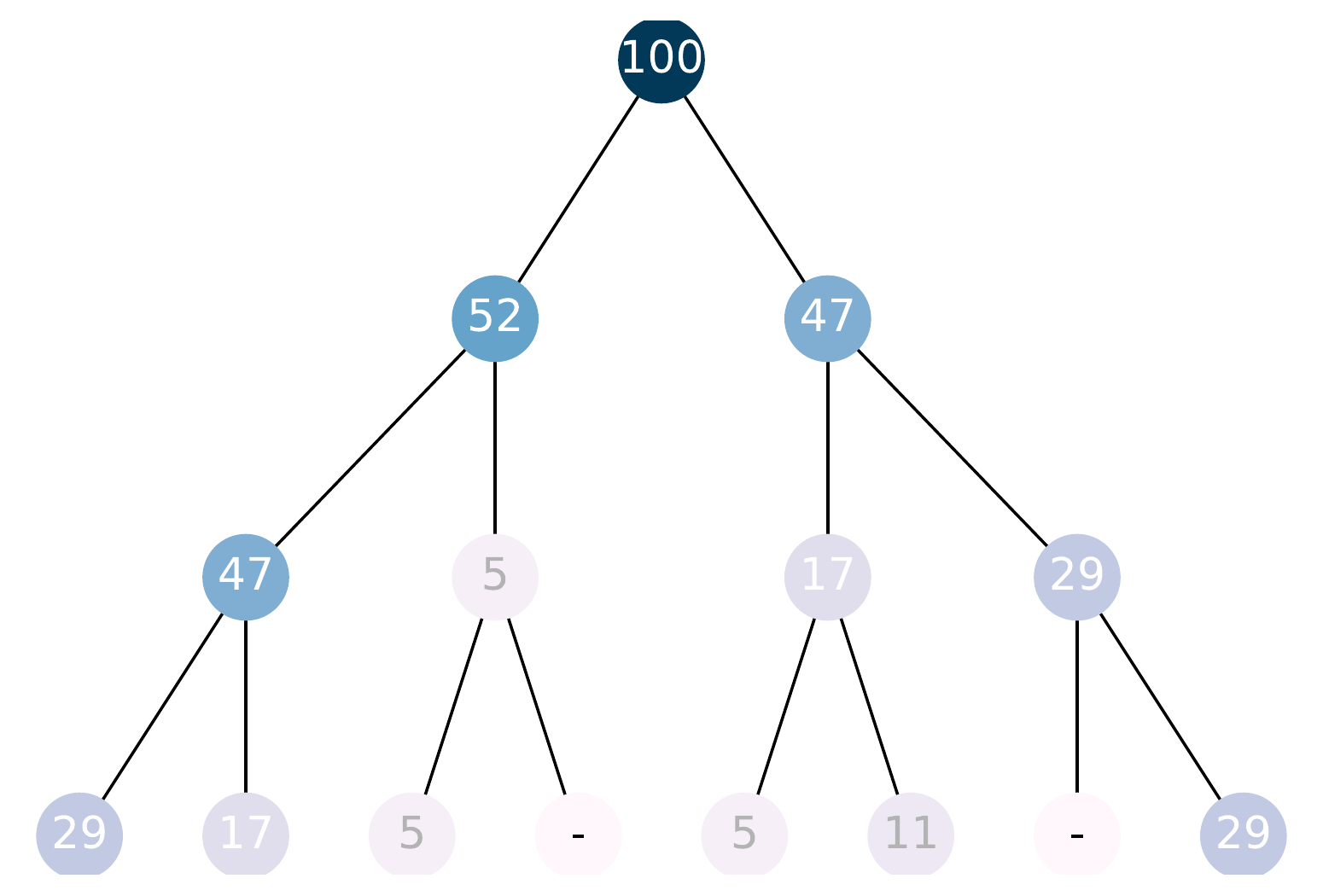}}\hfill
    \subfigure[NonW-Containers]{\label{fig:cont}\includegraphics[width=0.24\textwidth]{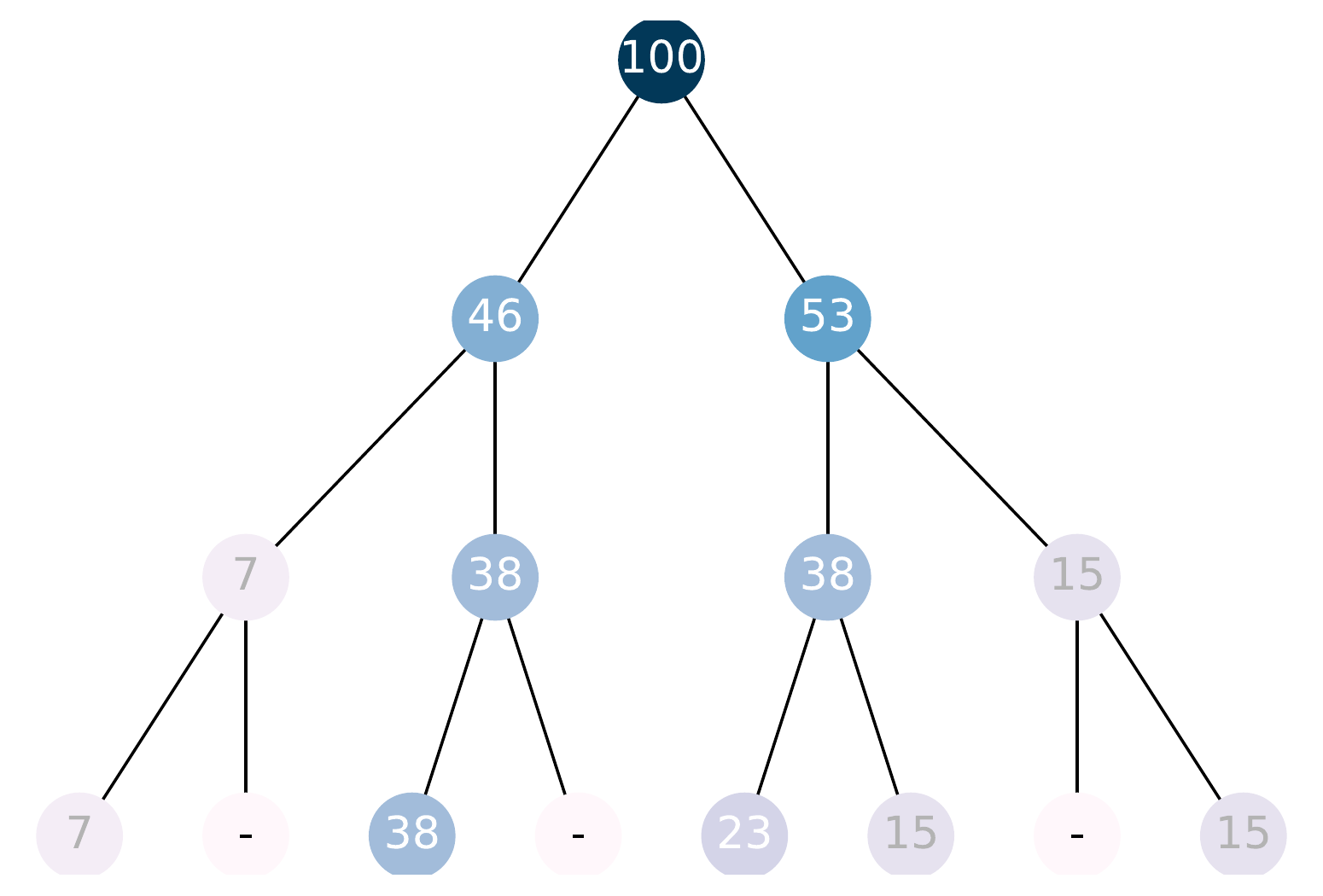}}\hfill
    \subfigure[NonW-Headlamps]{\label{fig:head}\includegraphics[width=0.24\textwidth]{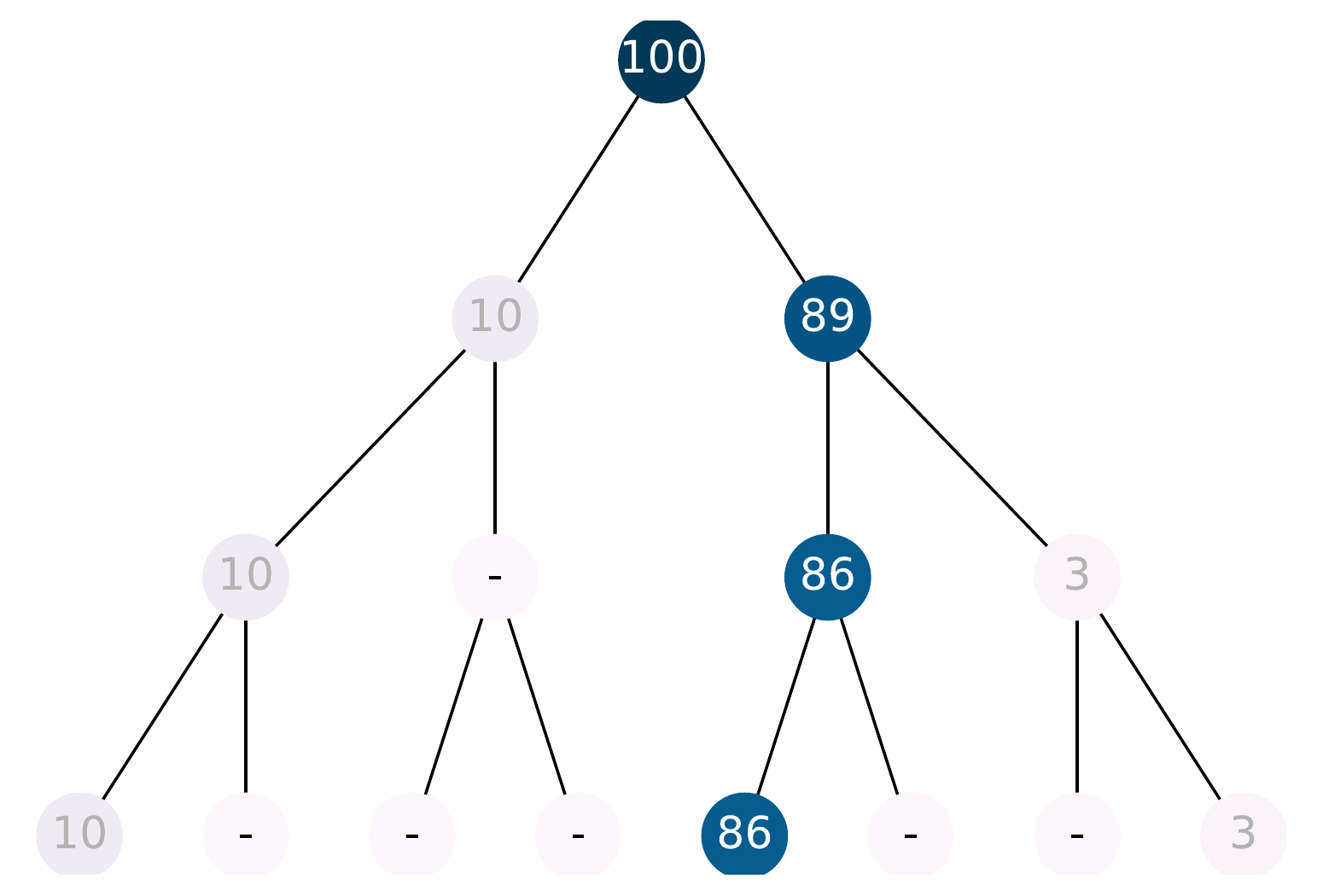}}
    \caption{\label{fig:class-routing}%
Class routing distributions on {\em Glass}, with distributions normalized over each depth level. Trees were trained with optimal hyper-parameters (depth $D\!=\!5$), but we plot nodes up to $D\!=\!3$ for visualization ease. Empty nodes are labeled by a dash $-$.
    }
\end{figure}


%% file: sections/discussion.tex
In this work we have presented a new optimization approach to learn trees for a variety of machine learning tasks. Our method works by sparsely relaxing a ILP for tree traversal and pruning, to enable simultaneous optimization of these parameters, alongside splitting parameters and downstream functions via argmin differentiation. Our approach nears or improves upon recent work in both supervised learning and hierarchical clustering. 
We believe there are many exciting avenues for future work. One particularly interesting direction would be to unify recent advances in tight relaxations of nearest neighbor classifiers with this approach to learn efficient neighbor querying structures such as ball trees. Another idea is to adapt this method to learn instance-specific trees such as parse trees.

%% file: sections/appendix.tex
\subsection{Motivation and derivation of Equation~\ref{eq:qp_prune}.}
\label{app:reg}

\paragraph{Smoothing and symmetric regularization.}
Consider the Heaviside function used to make a binary decision, defined as
\begin{equation}
   h(q) = \begin{cases} 1, & q \geq 0, \\ 0, & q < 0. \end{cases}
\end{equation}
To relax this discontinuous function into a continuous and differentiable one,
we remark that it is almost equivalent to the variational problem
\( \argmax \big\{zq : z \in \{0, 1\}\big\} \) with the exception of the ``tied'' case \(
q=0 \). We may relax the binary constraints into the convex interval $[0,1]$ and
add any strongly-concave regularizer to ensure the relaxation has a unique
maximizer everywhere and thus defines a continuous differentiable mapping:.
\begin{equation}
    \argmax_{0 \leq z \leq 1} zq - \omega(z).
\end{equation}
A tempting choice would be $\omega(z) = z^2/2$.
Note hovewer that as $|q|$ goes to zero, 
the problem converges to
$\argmax_{z} -z^2/2$ which is $0$.
This does not match our intuition for a good continuous relaxation, where a
small $q$ should suggest lack of commitment toward either option.
More natural here would be a regularizer that prefers
$z=1/2$ as $|q|$ goes to zero.
We may obtain such a regularizer by encoding the choice as a two-dimensional
vector $[z, 1-z]$ and penalizing its squared $l_2$ norm, giving:
\begin{equation}
\omega(z) = \frac{1}{4} \big\| [z, 1-z] \big\|^2
= \frac{1}{4} \big( z^2 + (1-z)^2 \big).
\end{equation}
As $|q| \to 0$, $\argmax_z -\omega(z)$ goes toward
$z=.5$. This regularizer, applied over an entire
vector $\sum_i \omega(z_i)$, is exactly the one we employ.
The $1/4$ weighting factor will simplify the derivation in the next paragraph.

\paragraph{Rewriting as a scaled projection.}
For brevity consider the terms for a single node, dropping the $i$ index:
\begin{align}
&\zb^\top \qb - \frac{1}{4} \big( \|\zb\|^2 + \|\bm{1} - \zb \|^2 \big)\\
=~&\zb^\top \qb - \frac{1}{4} \|\zb\|^2 -  \frac{1}{4} \|\bm{1}\|^2 - \frac{1}{4}
\|\zb\|^2 + \frac{1}{2} \zb^\top \bm{1}\\
=~&\zb^\top\underbrace{\left(\qb + \frac{1}{2} \bm{1} \right)}_{\tilde\qb} -\frac{1}{2} \|\zb\|^2
- \frac{1}{4} \|\bm{1}\|^2.
\intertext{We highlight and ignore terms that are constant w.r.t.\
$\zb$, as they have no effect on the optimization.}
=~&\zb^\top{\tilde \qb} - \frac{1}{2} \|\zb\|^2 + \text{const} \\
\intertext{We complete the square by adding and subtracting $.5\|\tilde\qb\|^2$, also constant
w.r.t. $\zb$.}
=~&\zb^\top{\tilde \qb} - \frac{1}{2} \|\zb\|^2 - \frac{1}{2} \|\tilde \qb\|^2 + \text{const} \\
=~& -\frac{1}{2} \| \zb - \tilde\qb \|^2 + \text{const}.
\end{align}
We may now add up these terms for every $i$ and flip the sign, turning the
maximization into a minimization, yielding the desired result.

To study how $\ab$ and $\zb$ are impacted by the relaxation, in Figure~\ref{fig:relax-diff} we report an empirical comparison of the solutions of the original Problem~\eqref{eq:ourilp} and of Problem~\eqref{eq:qp_prune}, after applying the quadratic relaxation.
For this experiment, we generated $1000$ random problems by uniformly drawing $\qb \in [-2, 2]^{n |\cT|}$, with $n = 10$ points and a tree of depth $D = 2$ and $|\cT| = 7$ nodes, and upper bounding the value of any child node with the one of its parent ($q_{it} = \min(q_{it}, q_{ip(t)})$), so that the tree constraints are satisfied.
We notice that the solutions of the relaxed problem are relatively close to the solutions of the original problem and that they tend to the optimal ones as the pruning hyper-parameter $\lambda$ increases.

\begin{figure}
    \centering
    \subfigure[Pruning variable $\ab$.]{\includegraphics[width=0.4\textwidth]{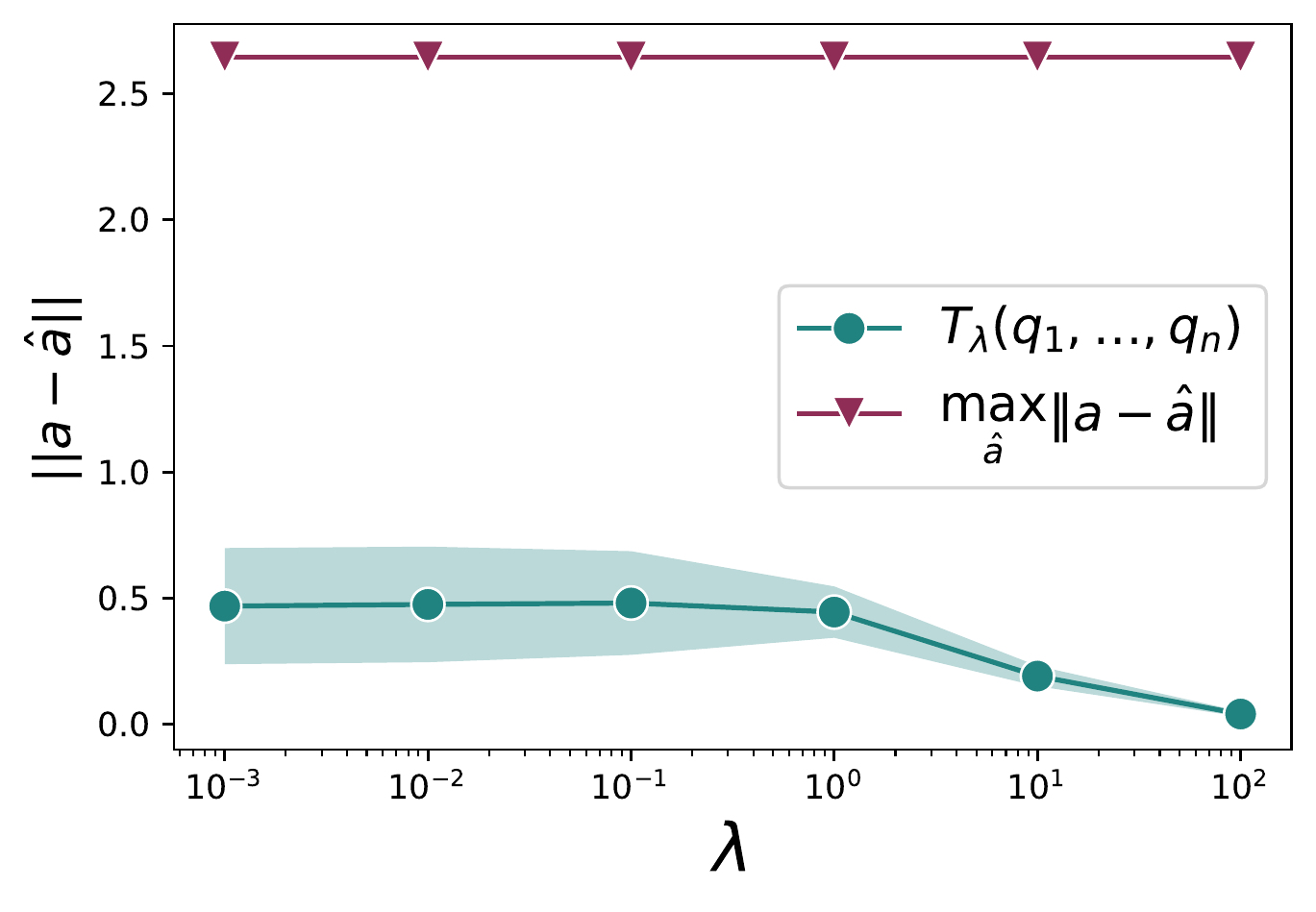}
    }\quad
    \subfigure[Tree traversals $\zb$.]{\includegraphics[width=0.4\textwidth]{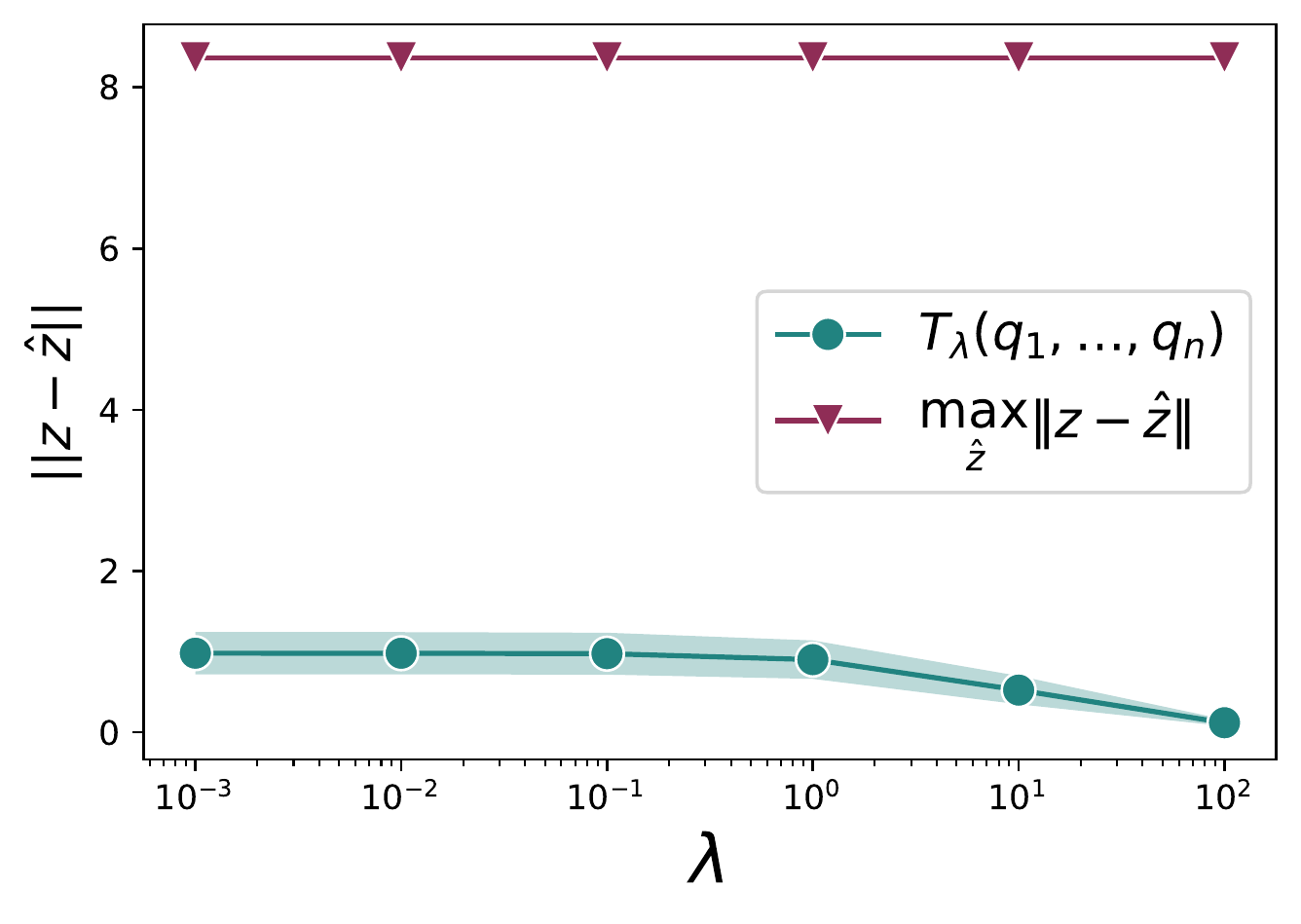}
    }
    \caption{Comparison of the solutions of Problem~\eqref{eq:ourilp}, denoted by $\ab$ and $\zb$, with the solutions of Problem~\eqref{eq:qp_prune}, denoted by $\hat{\ab}$ and $\hat{\zb}$.
    Average and one-standard-deviation intervals over $1000$ random problems are represented as a function of the pruning hyper-parameter $\lambda$, as well as the maximal possible gap. 
    }
    \label{fig:relax-diff}
\end{figure}


\subsection{Proof of Proposition~\ref{prop:scalar_subproblem}}\label{supp:scalar_subproblem}

Let $G \subset \mathcal{T}$ be a subset of (pooled) node indices.
We seek to solve
\begin{equation}\label{eq:subproblem-explicit}
\argmin_{a \in \bR} \sum_{t \in G} g_t(a)
~:=~
\argmin_{a \in [0, 1]} \frac{\lambda}{2}\sum_{t \in G}
 a^2
+ \sum_{i: a \leq q_{it}} \hlf {(a - q_{it} - \hlf)}^2
\end{equation}
Note that the final summation implicitly depends on the unknown $a$.
But, regardless of the
value of $a$, if $q_{it} \leq q_{i't'}$ and $q_{it}$ is included in the sum, then $q_{i't'}$ must also be included by transitivity.
We may therefore characterize the solution $a^\star$ via the number
of active inequality constraints
$k^\star = \big| \{(i, t) : a^\star \leq q_{i,t} + \hlf \}\big|$.
We do not know $a^\star$, but it is trivial to find by testing all possible values of $k$. For each $k$, we may find the set $S(k)$ defined in the proposition. For a given $k$, the candidate objective is
\begin{equation}
J_k(a) = \frac{\lambda}{2} \sum_{t \in G} a^2 + \sum_{(i,t) \in S(k)} \hlf(a - q_{it} - \hlf)^2
\end{equation}
and the corresponding $a(k)$ minimizing it can be found by setting the gradient
to zero:
\begin{equation}
J'_k(a)
= \lambda \sum_{t \in G} a
+ \sum_{(i,t) \in S(k)} (a - q_{i,t} - \hlf)
:= 0
\quad
\iff
\quad
a(k) = \frac{\sum_{(i,t) \in S(k)} (q_{it} + \hlf)}{\lambda|G|+k}.
\end{equation}
Since $|S(k)|=k$ and each increase in $k$ adds a non-zero term to the
objective, we must have $J_1\big(a(1)\big) \leq J_1\big(a(2)\big) \leq
J_2\big(a(2)\big) \leq \ldots$,
so we must take $k$ to be as small as possible,
while also ensuring the condition $|\{(i,t) : a(k) \leq q_{it} + \hlf\}| = k$,
or, equivalently, that $a(k) > q_{j^{([k+1])}} + \hlf$.
The box constraints may be integrated
at this point via clipping, yielding $a^\star =
\operatorname{Proj}_{[0,1]} \big(a(k^\star)\big)$.

\subsection{Benchmarking solvers}\label{app:time}
We study the running time and memory usage of Algorithm~\ref{alg} depending on the number of data points $n$ and the chosen tree depth $D$.
We compare its solving time and memory with those needed by the popular differentiable convex optimization framework {\bf Cvxpylayers}~\citep{cvxpylayers2019} (based on Cvxpy~\citep{diamond2016cvxpy, agrawal2018rewriting}) to solve Problem~\ref{eq:qp_prune}.
As this library is based on solvers implemented in Objective C and C we implement our approach in C++ for a fair comparison.
We report the solving times and memory consumption in Figure~\ref{fig:alg-time} for a forward-and-backward pass, where we vary one parameter $n$ or $D$ at a time and fix the other.
The algorithm that we specifically designed to solve problem~\eqref{eq:qp_prune} is indeed several magnitude faster than the considered generic solver and consumes significantly less memory, overall scaling better to the size of the tree and the number of inputs.

\begin{figure}
    \centering
    \subfigure[Fixed number of points $n = 100$]{\includegraphics[width=0.4\textwidth]{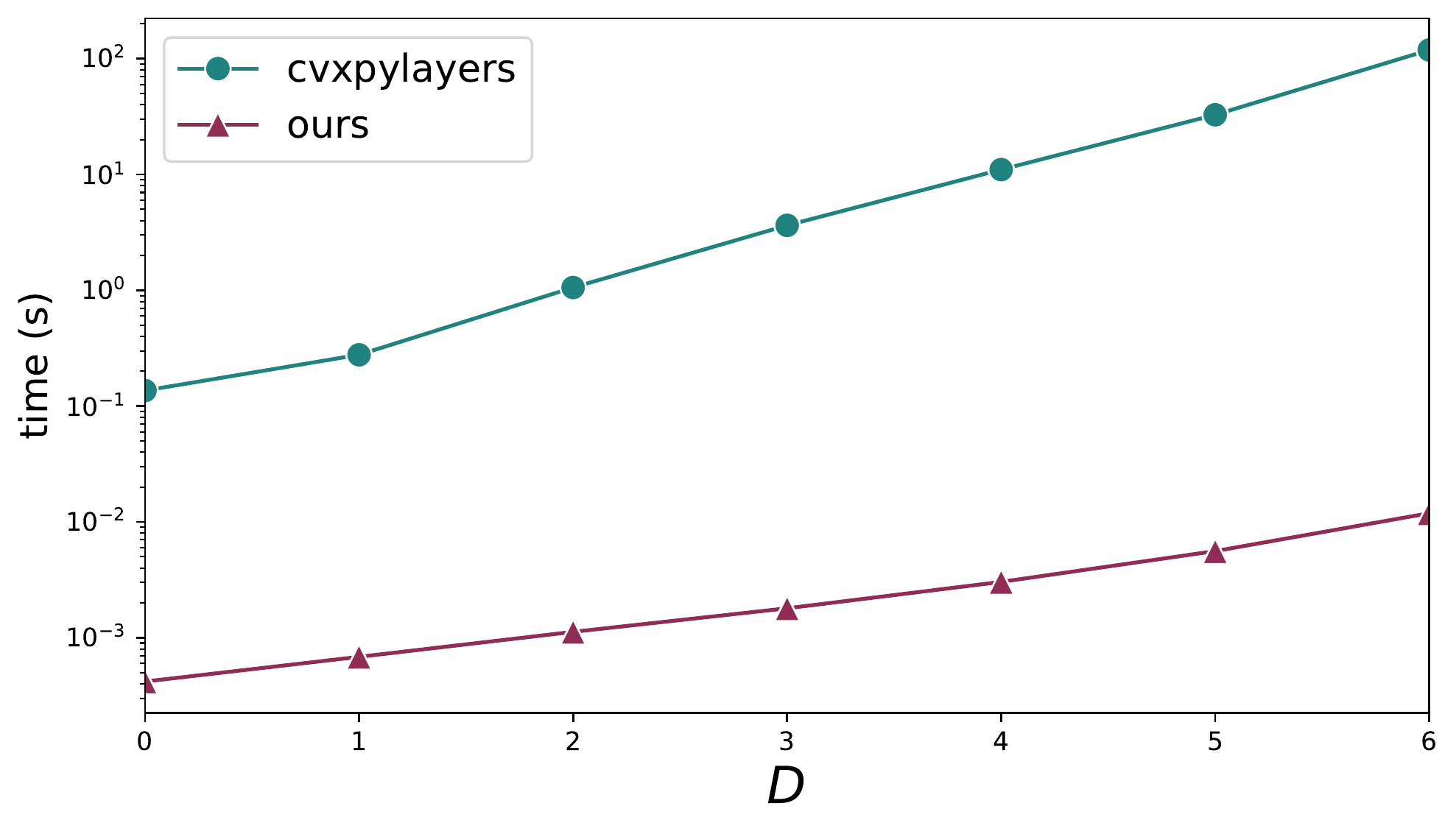}
    }\quad
    \subfigure[Fixed tree depth $D = 3$]{\includegraphics[width=0.4\textwidth]{figures/time-D3.pdf}
    }\\
    \subfigure[Fixed number of points $n = 100$]{\includegraphics[width=0.4\textwidth]{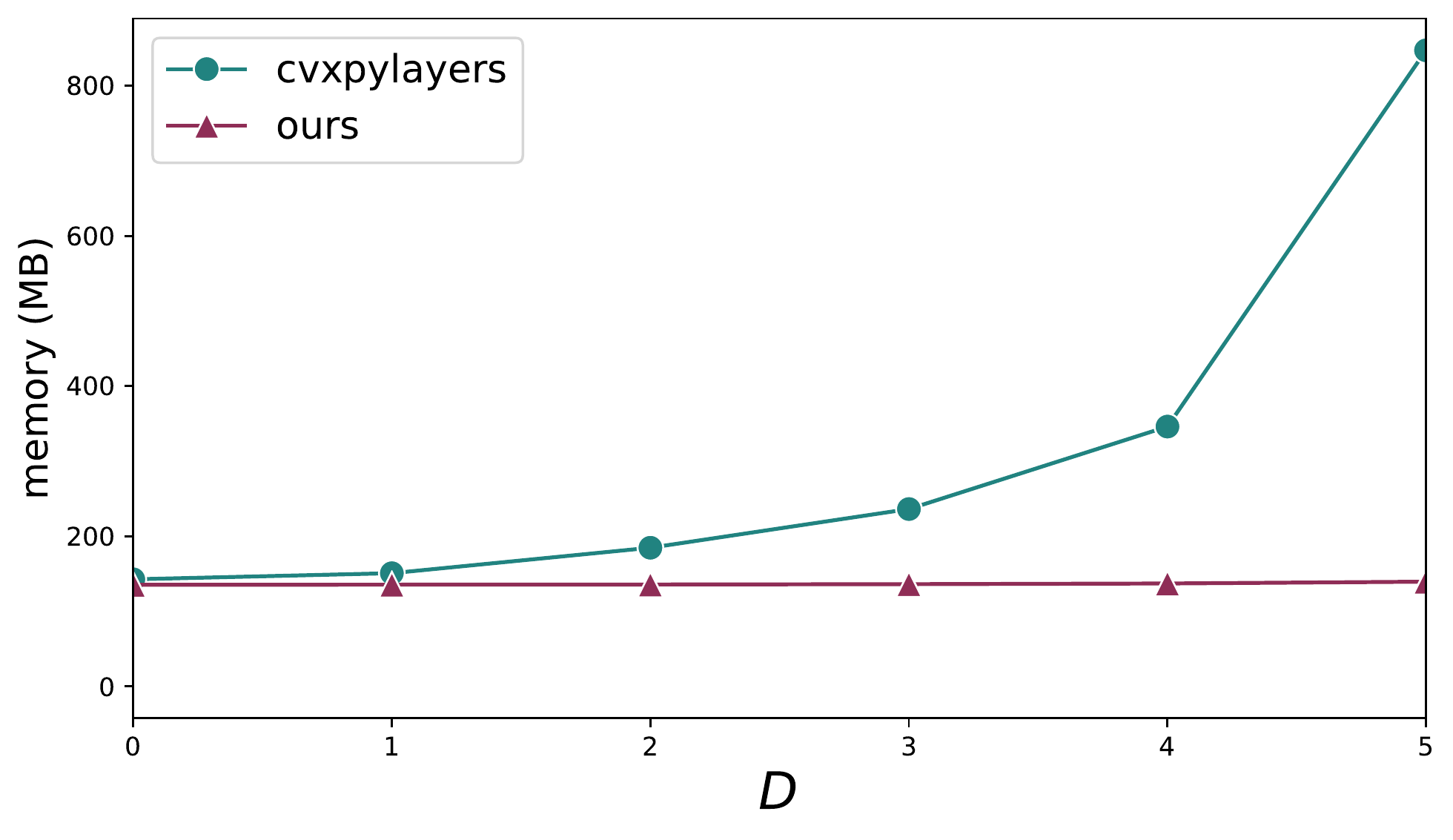}
    }\quad
    \subfigure[Fixed tree depth $D = 3$]{\includegraphics[width=0.4\textwidth]{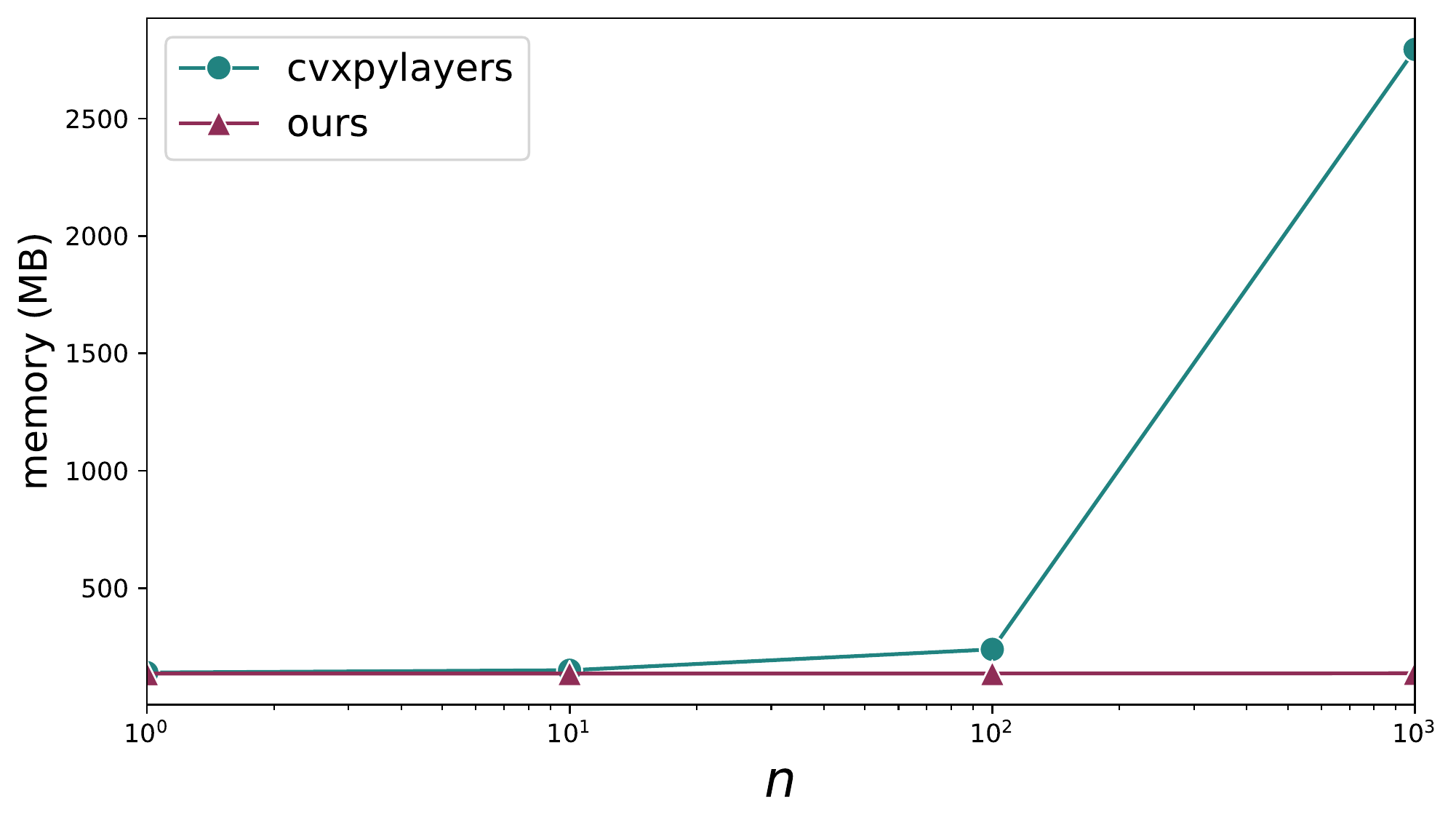}
    }
    \caption{Comparison of Algorithm~\ref{alg} and {\bf Cvxpylayers} in terms of (a-b) running time and (c-d) memory usage. $n$ takes values in a range that covers common choices of batch size. Time and $n$ axis are represented in logarithmic scale.
    }
    \label{fig:alg-time}
\end{figure}

\subsection{Further experimental details}
We tune the hyper-parameters of all methods with Optuna Bayesian optimizer~\citep{optuna_2019}, fixing the number of trials to $100$.
For all back-propagation-based methods, we fix the learning rate to $0.001$ and the batch size to $512$, use a scheduler reducing this parameter of a factor of $10$ every $2$ epochs where the validation loss has not improved, and fix the maximal number of epochs to $100$ and patience equal to $10$ for early stopping.
For our method, we further initialize the bias $\bm{b} \in \mathbb{R}^{|\mathcal{T}|}$ for the split function $s_\theta(x)$ (explicitly defining $s_\theta(x) = s_{\theta \setminus \bm{b}}(x) + \bm{b}$) to ensure that points are equally distributed over the leaves at initialization.
We hence initialize the bias to minus the initial average value of the training points traversing each node: $b_t = - \frac{1}{|\{x_i|q_{it} > 0\}_{i=1}^n|}\sum_{i=1}^n s_{{\theta_t \setminus b_t}}(x_i) \mathbbm{1}{[q_{it} > 0]}$.

\paragraph{Experiments on tabular datasets}
The other hyper-parameters for these experiments are chosen as follows:
\begin{itemize}
    \item {\bf Ours}: we tune $D$ in $\{2, \dots, 6\}$, $\lambda$ in $[1\text{e-}3, 1\text{e+}3]$ with log-uniform draws, the number of layers of the MLP $L$ in $\{2, \dots, 5\}$ and its dropout probability uniformly in $[0, 0.5]$, and the choice of activation for the splitting functions as linear or Elu;

    \item {\bf Optree-LS}: we fix the tree depth $D = 6$;

    \item {\bf CART}: we not bound the maximal depth and tune \textit{feature rate} uniformly in $[0.5, 1]$, \textit{minimal impurity decrease} in $[0, 1]$, minimal Cost-Complexity Pruning $\alpha$ log-uniformly in $[1\text{e-}16, 1\text{e+}2]$ and \textit{splitter} chosen between best or random;

    \item {\bf Node} and {\bf XGBoost}: all results are reported from \citet{popov2019neural}, where they used the same experimental set-up;

    \item {\bf DNDT}: we tune the softmax temperature uniformly between $[0, 1]$ and the number of feature cuts in $\{1, 2\}$;

    \item {\bf NDF}: we tune $D$ in $\{2, \dots, 6\}$ and fix the feature rate to $1$;

    \item {\bf ANT}: as this method is extremely modular (like ours), to allow for a fair comparison we choose a similar configuration. We then select as transformer the identity function, as router a linear layer followed by the Relu activation and with soft (sigmoid) traversals, and as solver a MLP with $L$ hidden layers, as defined for our method; we hence tune $L$ in $\{2, \dots, 5\}$ and its dropout probability uniformly in $[0, 0.5]$, and fix the maximal tree depth $D$ to $6$; we finally fix the number of epochs for growing the tree and the number of epochs for fine-tuning it both to $50$.

\end{itemize}

\paragraph{Experiments on hierarchical clustering}
For this set of experiments, we make us of a linear predictor and of linear splitting functions without activation.
We fix the learning rate to $0.001$ and the batch size $512$ for {\em Covtype} and $8$ for {\em Glass}.
The other hyper-parameters of our method are chosen as follows:
we tune $D$ in $\{2, \dots, 6\}$, $\lambda$ in $[1\text{e-}3, 1\text{e+}3]$ with log-uniform draws.
The results of the baselines are reported from \citet{monath2019gradient}.

\subsection{Additional experiments}
In Figure~\ref{fig:times} we represent the average test Error Rates or Mean Square Error as a function of the training time for each single-tree method on the tabular datasets of Section~\ref{sec:tabular}.
Notice that our method provides the best trade-off between time complexity and accuracy over all datasets.
In particular, it achieves Error Rates comparable on {\em Click} and significantly better on {\em Higgs} w.r.t. {\bf NDF-1} while being several times faster.
Table~\ref{tab:sizes} shows that this speed-up is not necessarily due to a smaller number of model parameters, but it comes from the deterministic routing of points through the tree.
Despite having model sizes bigger than ANT's ones, our method is significantly faster than this baseline as it offers an efficient way for optimizing the tree structure (via the joint optimization of pruning vector $\bm{a}$).
In comparison, ANT needs to incrementally grow trees in a first phase, to then fine-tune them in a second phase, resulting in a computational overhead.
Finally, {\em DNDT} implements a soft binning with trainable binning thresholds (or cut points) for each feature, hence scaling badly w.r.t this parameter and resulting in a memory overflow on {\em HIGGS}.

\begin{table}[t!]
\caption{Number of parameters for single-tree methods on tabular datasets.}

\label{tab:sizes}
\begin{center}
\scalebox{0.85}{
\begin{tabular}{clccccc}  \hline
 &METHOD  & YEAR & MICROSOFT & YAHOO & CLICK & HIGGS \\ \hline
\parbox[t]{2mm}{\multirow{5}{*}{\rotatebox[origin=c]{90}{Single Tree}}}
& CART & $164$ & $58$ & $883$ & $12$ & $80$\\
& DNDT    & - & - & - & $4096$ & -\\
& NDF-1   & - & - & - & $78016$ & $47168$\\
& ANT & $179265$ & $17217$ & $53249$ & $81$ & $7874$\\
& Ours    & $186211$ & $52309$ & $55806$ & $43434$ & $701665$ \\ \hline
\end{tabular}}
\end{center}

\end{table}

\begin{figure}
    \subfigure[{\bf Click}]{    \label{fig:time-click}\includegraphics[width=0.47\textwidth]{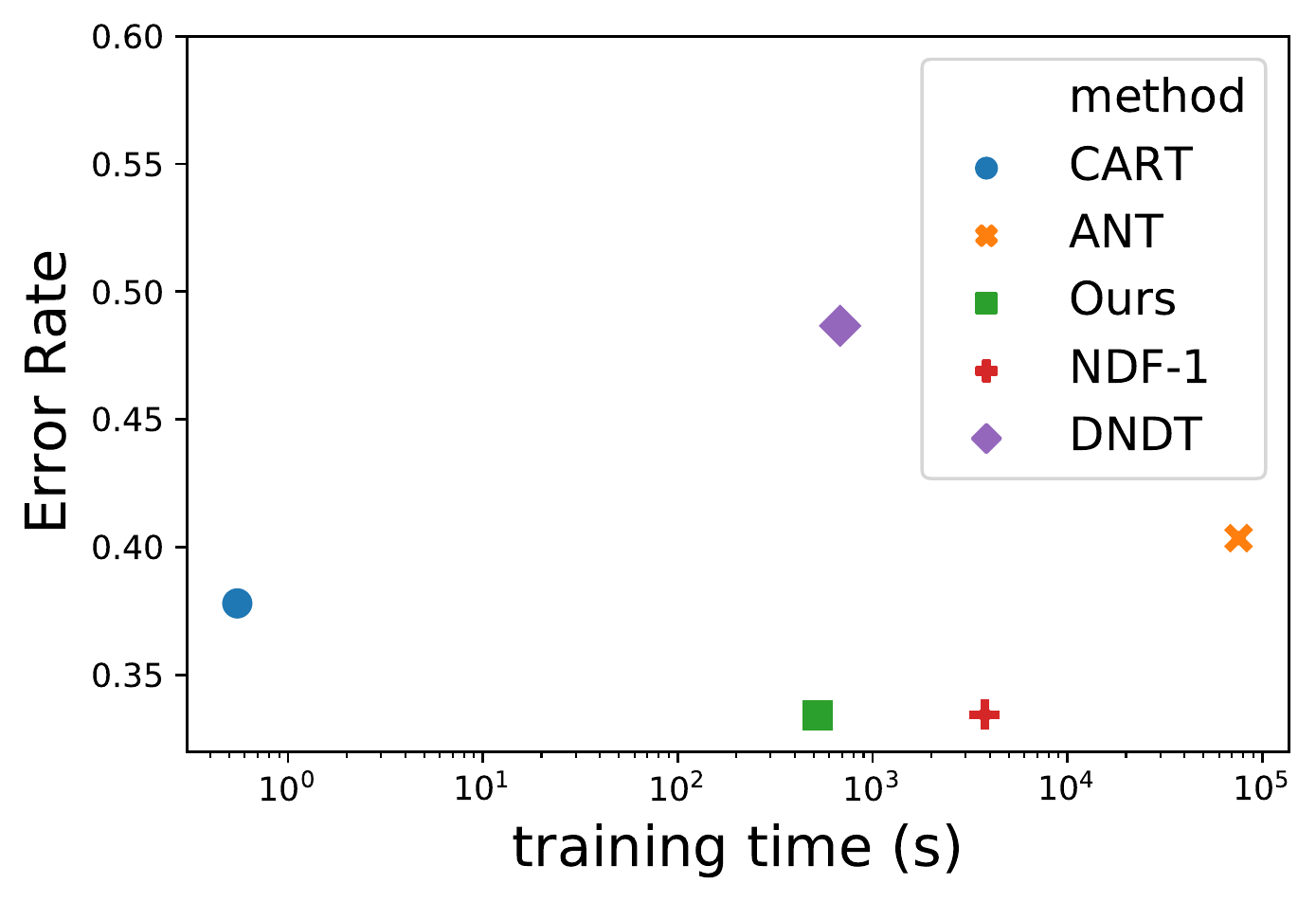}
    }\hfill
    \subfigure[{\bf Higgs}]{    \label{fig:time-higgs}\includegraphics[width=0.47\textwidth]{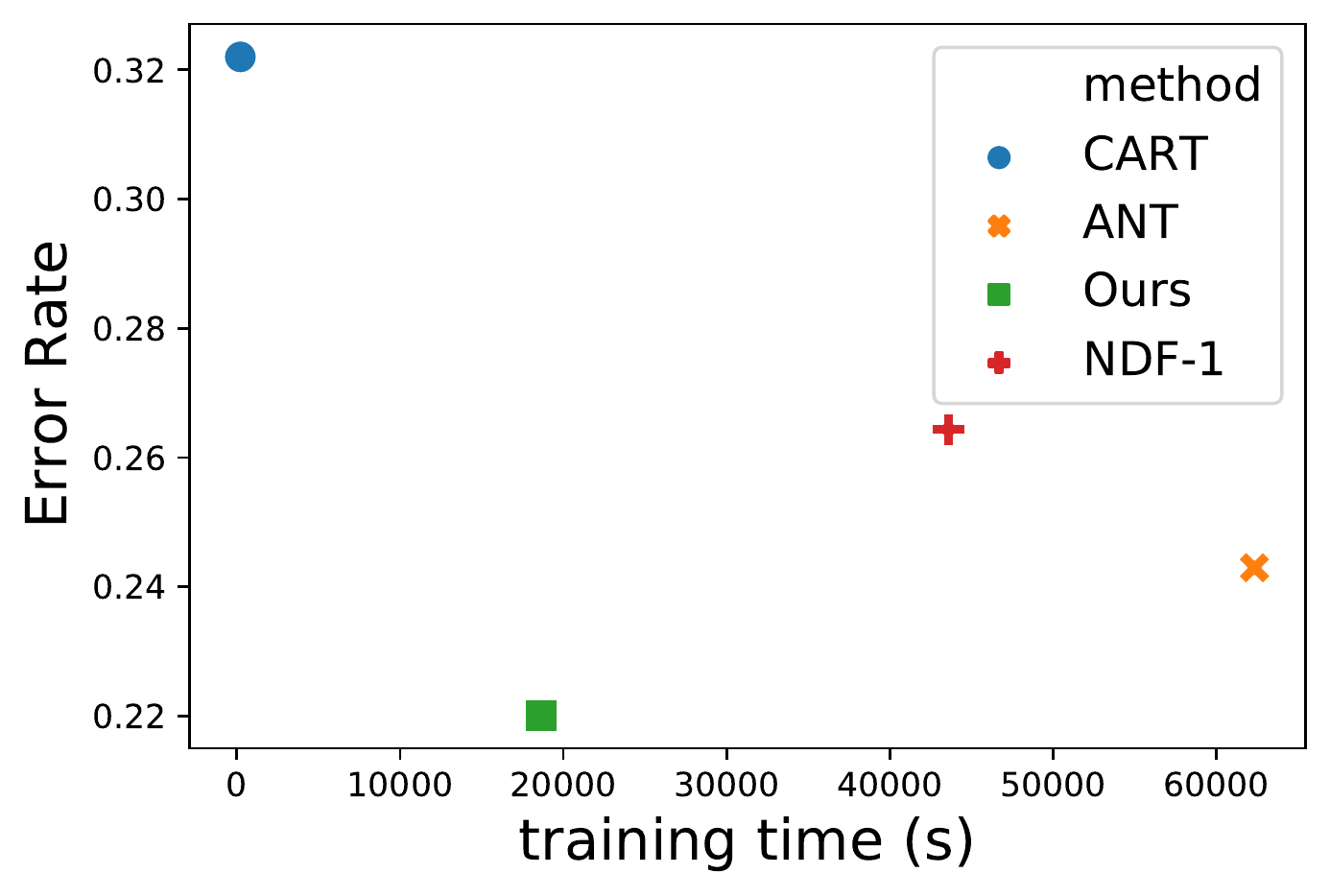}
    }\\
    \subfigure[{\bf Year}]{    \label{fig:time-year}\includegraphics[width=0.47\textwidth]{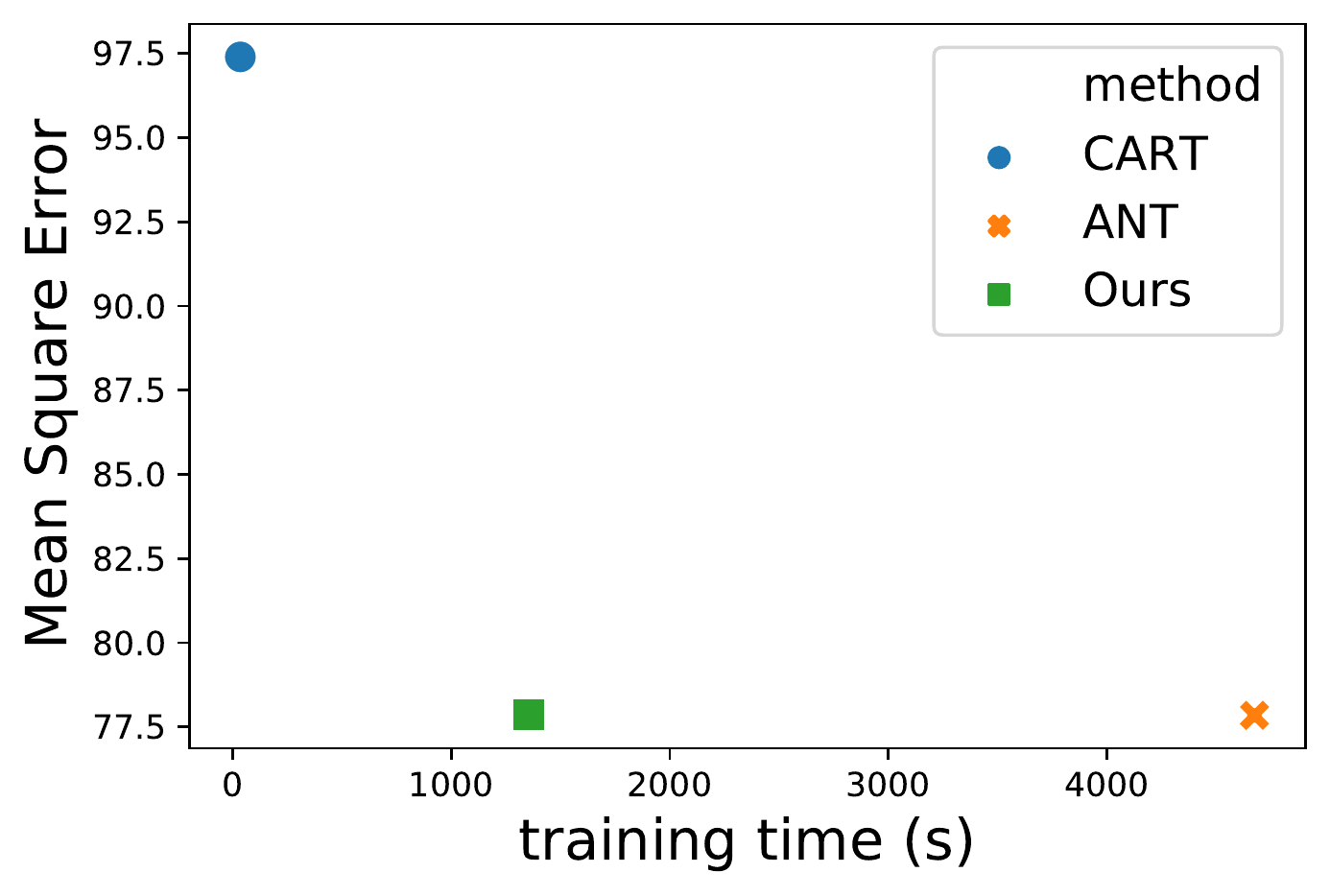}
    }\hfill
    \subfigure[{\bf Microsoft}]{    \label{fig:time-mic}\includegraphics[width=0.47\textwidth]{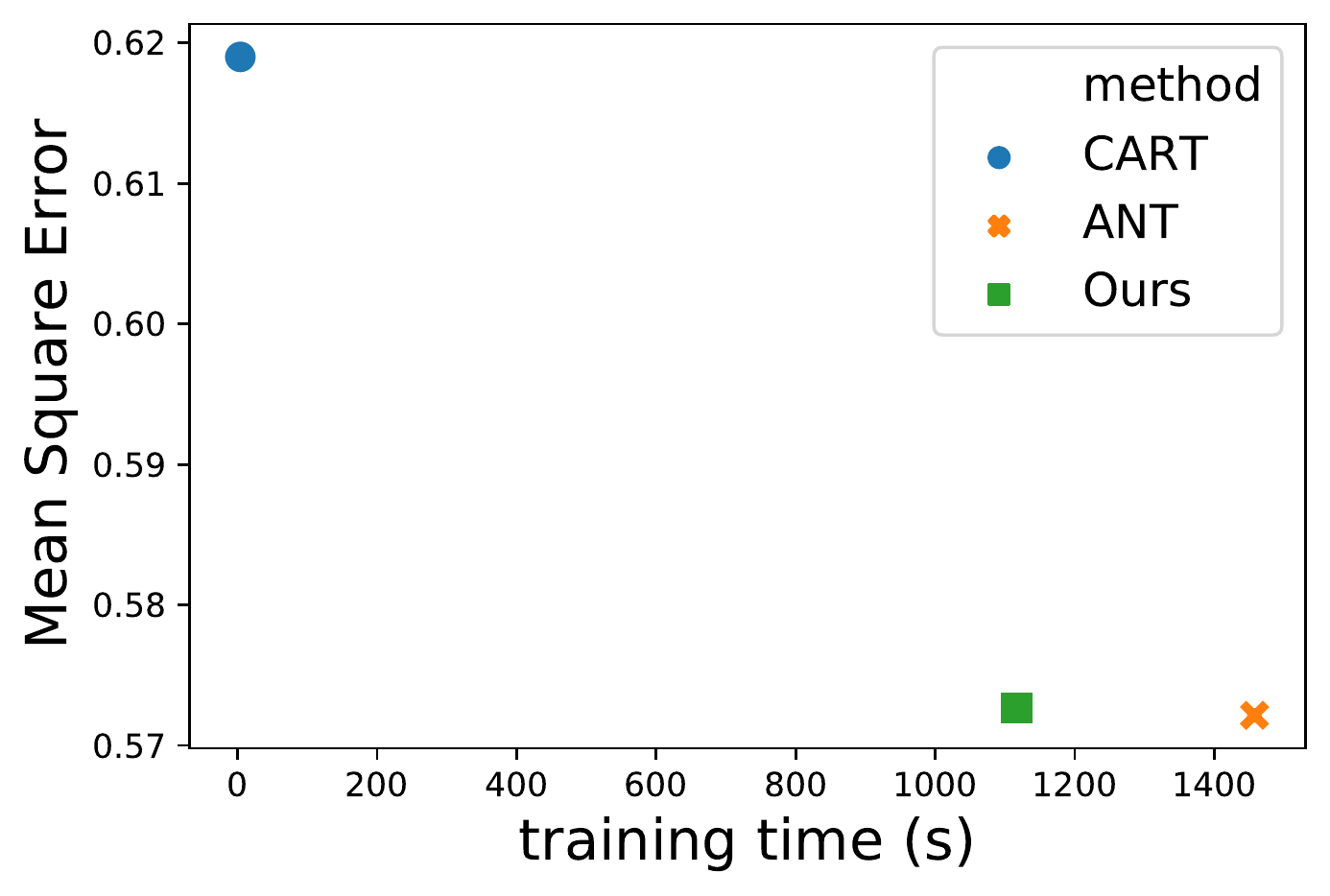}
    }\\
    \subfigure[{\bf Yahoo}]{    \label{fig:time-yahoo}\includegraphics[width=0.47\textwidth]{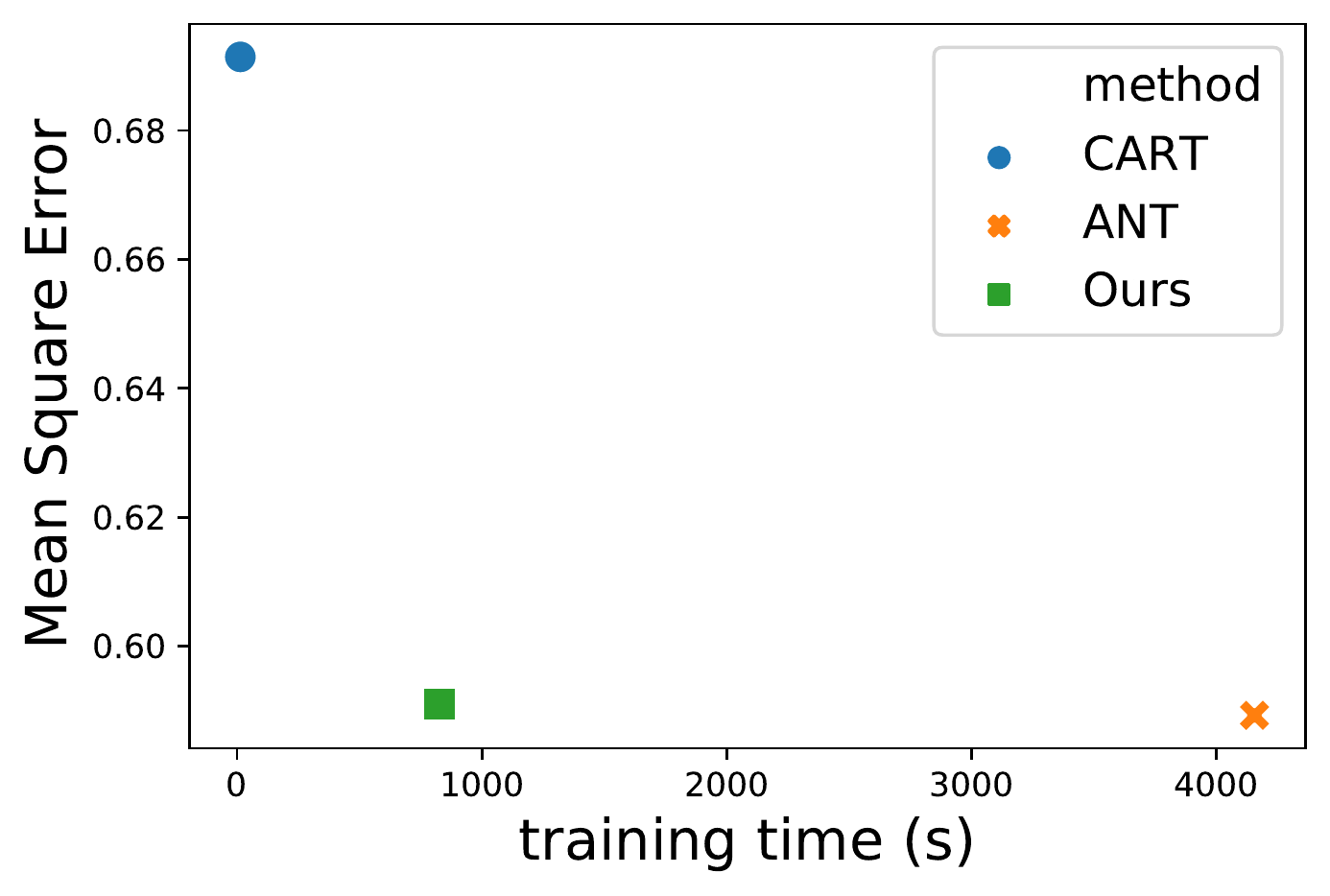}
    }\hfill
    \caption{Average (a,b) Error Rate (c-e) Mean Square Error vs average training time required by each method.}
    \label{fig:times}
\end{figure}

\paragraph{Further comparison with optimal tree baselines}
We run a set of experiments on small binary classification datasets to compare our method with optimal tree methods, which do not scale well with the dataset size.
Specifically we compare against two versions of Optree: one that solves the MIP exactly (Optree-MIP) \citep{bertsimas2017optimal}, and another that solves it with local search Optree-LS \citep{dunn2018optimal}.
We also compare with the state-of-the-art optimal tree method GOSDT
and to CART.
We consider the binary classification datasets
\begin{itemize}
    \item {\em Mushrooms}~\citep{schlimmer1987concept}: prediction between edible and poisonous mushrooms, with $8124$ instances and $22$ features;
 \item {\em FICO}\footnote{\url{https://community.fico.com/s/explainable-machine-learning-challenge}}: loan approval prediction, with $10459$ instances and $24$ features;
\item {\em Tictactoe}\footnote{\url{https://archive.ics.uci.edu/ml/datasets/Tic-Tac-Toe+Endgame}}: endgame winner prediction, with $958$ instances and $9$ features.
\end{itemize}
We apply a stratified split on all datasets to obtain $60\%$-$20\%$-$20\%$ training-validation-test sets,  convert categorical features to ordinal, and z-score them. 
For our method, we apply the Quasi-Hyperbolic Adam optimization algorithm, with batch size equal to $512$ for {\em Mushrooms} and {\em FICO}, and $64$ for {\em Tictactoe}.

We chose the hyper-parameters as follows:
\begin{itemize}
    \item {\bf Ours}: we tune $D$ in $\{2, \dots, 6\}$, $\lambda$ in $[1\text{e-}3, 1\text{e+}3]$ with log-uniform draws, and make use of a linear predictor and of linear splitting functions without activation;
    \item {\bf Optree-MIP}: we fix the tree depth to $D = 6$;
    \item {\bf Optree-LS}: we tune $D$ in $\{2, \dots, 10\}$;
    \item {\bf GOSDT}:  we tune the regularization parameter $\lambda$ in $[5\text{e-}3, 1\text{e+}3]$ with log-uniform draws, and set accuracy as the objective function.
    \item {\bf CART}: we tune $D$ in $\{2, \dots, 20\}$, \textit{feature rate} uniformly in $[0.5, 1]$, \textit{minimal impurity decrease} in $[0, 1]$, $\alpha$ log-uniformly in $[1\text{e-}16, 1\text{e+}2]$ and \textit{splitter} chosen between best or random.
\end{itemize}

Table~\ref{tab:small} reports the performance of all methods across $4$ runs.
Both OPTREE variants do not finish in under $3$ days on \emph{Mushrooms} and {\em FICO}, which have more than $1000$ instances.
CART is the fastest method and ranks second in terms of error rate on all datasets but {\em Tictactoe}, where it surprisingly beats all non-greedy baselines.
Our method outperforms optimal tree baselines in terms of error rate and training time on all datasets.
We believe the slow scaling of GOSDT is due to the fact that it binarizes
features, working with $118$, $1817$, $27$ final binary attributes respectively on {\em Mushrooms}, {\em FICO} and {\em Tictactoe}.
Apart from achieving superior performance on supervised classification tasks, we stress the fact that our method generalizes to more tasks, such as supervised regression and clustering as shown in the main text.

\begin{table}[t!]
\caption{Results on the Mushrooms tabular dataset. We report average training time (s), and average and standard deviations of test error rate over $4$ runs for binary classification datasets. We \textbf{bold} the best result (and those within a standard deviation from it). Experiments are run on a machine with 16 CPUs and 63GB  RAM. 
We limit the training time to $3$ days.}
\label{tab:small}
\begin{center}
\scalebox{0.9}{
\begin{tabular}{lcccccc} \hline
 METHOD &\multicolumn{2}{c}{MUSHROOMS}&\multicolumn{2}{c}{FICO} &\multicolumn{2}{c}{ TICTACTOE}\\
 & training time & error rate & training time & error rate & training time & error rate \\ \hline
CART & $0.004$s & $\mathbf{0.0003 \pm 0.0003}$ & $0.01$s & $0.3111 \pm 0.0042$ & $0.001$s & $\bm{0.1576 \pm 0.0203}$ \\
OPTREE-MIP & OOT & - & OOT & - & OOT & -\\
OPTREE-LS & OOT & - & OOT & - & $751$s & $0.449 \pm 0.0184$\\
GOSDT  & $214$s & $0.0125 \pm 0.0027$ & $634$s & $0.3660 \pm 0.0090$ & $40$s & $0.  3490 \pm 0.0010$\\
Ours    & $20$s & \bm{$0.0005 \pm 0.0008$} & $10$s & $\bm{0.2884 \pm  0.0020}$ & $13$s & $0.2669 \pm  0.0259$\\ \hline
\end{tabular}}
\end{center}

\end{table}